%% file: iclr2027_conference.tex
\documentclass{article} 
\usepackage{iclr2027_conference,times}

\input{math_commands.tex}

\usepackage{booktabs}
\usepackage{multirow}
\usepackage{makecell}
\usepackage{tabularx}
\usepackage{array}
\usepackage{graphicx}
\usepackage{pifont}
\usepackage[table]{xcolor}

\newcolumntype{C}[1]{>{\centering\arraybackslash}p{#1}}
\usepackage{hyperref}
\usepackage{url}
\usepackage{graphicx} 
\usepackage{booktabs}
\usepackage[table]{xcolor}
\usepackage{amsthm}
\usepackage{multicol}
\usepackage{multirow}
\newtheorem{theorem}{Theorem}
\usepackage{wrapfig}
\usepackage{booktabs}
\usepackage{multirow}
\usepackage{pifont}
\usepackage{wrapfig}
\usepackage{makecell}
\usepackage{array}
\newcolumntype{C}[1]{>{\centering\arraybackslash}m{#1}}
\renewcommand{\iclrruler}[1]{}

\title{SpikingVLA: Asynchronous Spiking Vision-Language-Action Models}

\author{\textbf{Jingya Wang}$^{1}$, \textbf{Dehao Zhang}$^{1}$, \textbf{Shuai Wang}$^{1}$, \textbf{Malu Zhang}$^{1,2}$\thanks{Corresponding author: maluzhang@uestc.edu.cn}, \textbf{Yang Yang}$^{1}$, \textbf{Haizhou Li}$^{2,3}$ \\ [0.3em]
  $^{1}$University of Electronic Science and Technology of China,\\ [0.3em]
    $^{2}$Shenzhen Loop Area Institute,
  $^{3}$The Chinese University of Hong Kong (Shenzhen) 
}

\begin{document}

\maketitle

\begin{abstract}
ANN-to-SNN conversion offers a practical route toward energy-efficient spiking Vision-Language-Action (VLA) models by bypassing the substantial cost of training large-scale SNNs from scratch. However, existing methods often require many timesteps to maintain competitive performance, resulting in substantial inference latency for real-time VLA deployment. To address this challenge, we introduce SpikingVLA, an ANN-to-SNN conversion framework that enables accurate and low-latency spiking VLA inference. Specifically, we propose a Dendritic Integrate-and-Fire (DIF) neuron that alleviates channel-wise activation outliers through dendritic mixing and adaptive somatic firing, enabling accurate ANN-to-SNN conversion with fewer timesteps. Building on DIF neurons, we further introduce an asynchronous execution mechanism that overlaps temporal computation across VLA components, reducing synchronization overhead and latency. Extensive experiments demonstrate that SpikingVLA achieves competitive navigation performance with substantially improved inference efficiency. Compared with existing spiking VLA methods, SpikingVLA improves SR and SPL by 11.9\% and 12.6\%, respectively, while reducing first-action latency by 11.2$\times$. These results establish SpikingVLA as a practical framework for deploying pretrained VLA models with high-performance and low-latency spiking inference.
\end{abstract}

\section{Introduction}
Vision-Language-Action (VLA) models extend large-scale vision-language models to embodied decision making by grounding multimodal observations and language instructions into   actions~\citep{brohan2023rt,team2025gemini}. Leveraging the perception and reasoning capabilities acquired through large-scale multimodal pretraining, VLA models have demonstrated strong performance on a broad range of embodied tasks, including robotic manipulation and navigation~\citep{black2024pi_0,cheng2024navila}. However, their reliance on large vision-language backbones incurs substantial computational overhead, hindering deployment on resource-constrained robotic platforms~\citep{kim2025fine,shukor2025smolvla}. Although recent studies have explored model-compression techniques to reduce these resource demands~\citep{xu2026qvla}, these methods generally preserve dense execution, leaving substantial redundant computation during inference~\citep{xie2026turbovla, yue2024deer}. This motivates computational paradigms that directly reduce redundant inference operations.

Spiking Neural Networks (SNNs) provide an event-driven computing paradigm, where information is represented through discrete spike events~\citep{maass1997networks,izhikevich2003simple,gerstner2002spiking}. By exploiting sparse neural activity, SNNs can reduce redundant computation and energy consumption during inference~\citep{deng2020tianjic,davies2018loihi,deng2020rethinking,ma2017darwin}. However, training large-scale SNNs from scratch remains challenging due to the difficulty of optimizing temporal spiking dynamics~\citep{wu2018spatio, zhang2026dendritic, fang2023parallel}. ANN-to-SNN conversion provides a practical alternative by transforming pretrained ANNs into spiking models while preserving their learned representations~\citep{rueckauer2017conversion,li2024error}. By leveraging pretrained knowledge, these conversion methods avoid costly end-to-end SNN training and enable scalable adaptation of large neural models to event-driven spiking computation~\citep{li2021free,bu2023optimal,wang2026training,bu2025inference}.

\begin{figure}[tbp]
    \centering
    \includegraphics[width=0.98\linewidth]{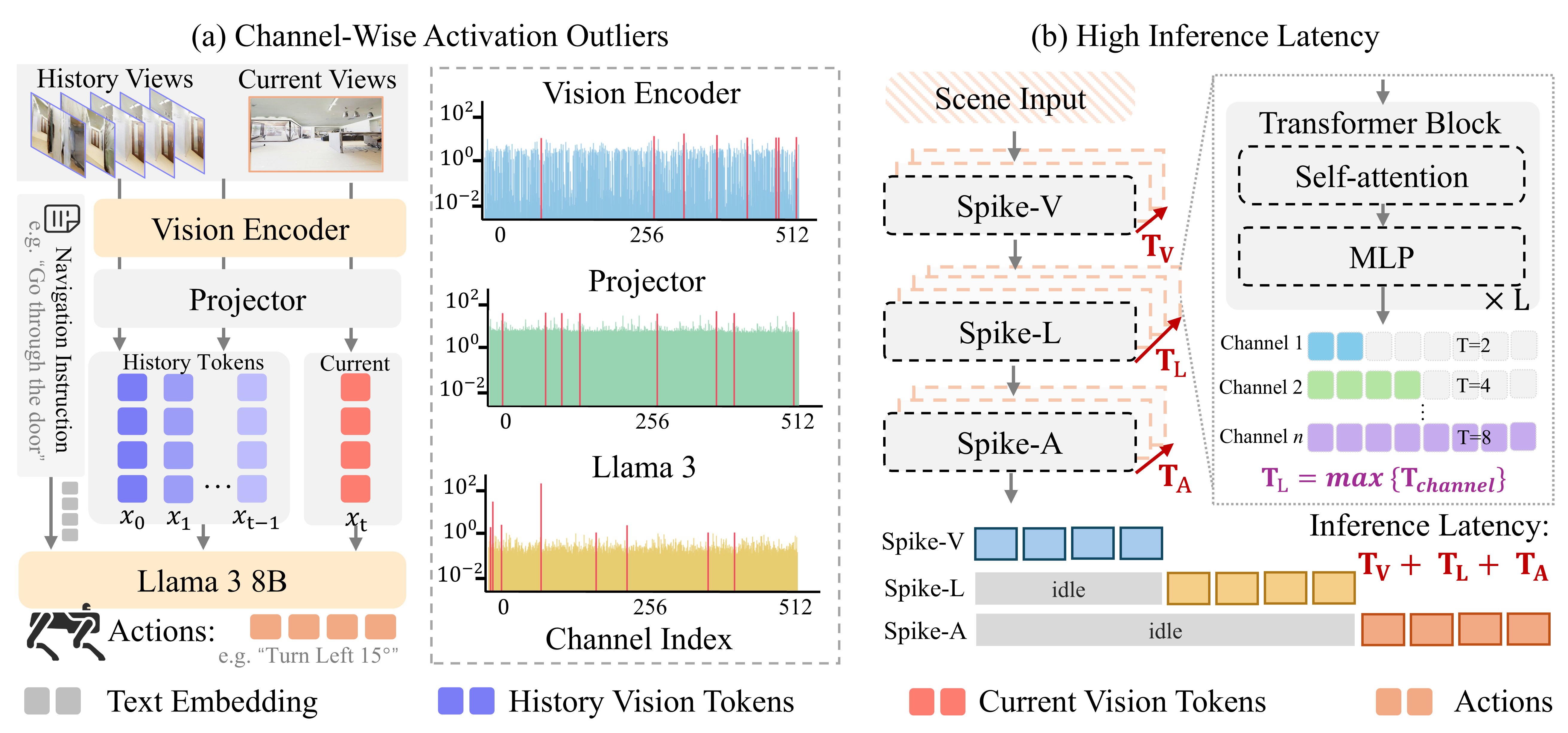}
    \caption{(a) Pronounced activation outliers across VLA components increase spike encoding errors at limited timesteps, degrading conversion accuracy. (b) Non-uniform temporal horizons across components introduce synchronization stalls and idle periods, increasing inference latency.}
    \label{fig:overview}
    \vspace{-15pt}
\end{figure}

Recent work has demonstrated the feasibility of ANN-to-SNN converted VLA models for embodied navigation~\citep{song2026spikevla}. However, scaling spiking VLA models to full VLA systems while preserving task performance and inference efficiency remains challenging. First, VLA components exhibit pronounced channel-wise activation outliers. As shown in Fig.~\ref{fig:overview}(a), the vision encoder, projector, and language backbone present substantially different activation distributions, where a small subset of channels dominates the activation range. These outliers increase spike encoding errors under limited timesteps, degrading conversion accuracy. Moreover, VLA components exhibit non-uniform temporal horizons during spiking inference. As shown in Fig.~\ref{fig:overview}(b), the vision, language, and action branches operate over $T_{\mathrm{V}}$, $T_{\mathrm{L}}$, and $T_{\mathrm{A}}$ timesteps, respectively. These differences in temporal horizons force downstream components to wait for upstream completion, introducing synchronization stalls and increasing end-to-end inference latency.

To address these limitations, we introduce SpikingVLA, an efficient ANN-to-SNN conversion framework for accurate and low-latency embodied navigation. Specifically, we propose a Dendritic Integrate-and-Fire (DIF) neuron that mitigates channel-wise activation outliers through dendritic mixing and adaptive somatic firing, enabling accurate conversion with fewer timesteps. Building on DIF, we further develop an asynchronous execution mechanism that enables VLA components to process intermediate spike responses without requiring full temporal synchronization, reducing synchronization overhead and inference latency. Extensive experiments on vision-language navigation benchmarks and closed-loop robotic navigation demonstrate that SpikingVLA maintains near-lossless navigation performance relative to its ANN counterpart while substantially improving inference efficiency. Our main contributions are summarized as follows:
\begin{itemize}
    \item We identify two key bottlenecks in spiking VLA models: pronounced activation outliers degrade conversion accuracy at limited timesteps, while non-uniform temporal horizons across VLA components introduce synchronization stalls and increase inference latency.
    \item We introduce a DIF-based framework for accurate and low-latency spiking VLA inference. DIF mitigates channel-wise activation outliers through dendritic redistribution and adaptive somatic firing, enabling accurate conversion with fewer timesteps. Building on DIF, we further develop an asynchronous execution mechanism that propagates intermediate responses across VLA components, reducing synchronization stalls and inference latency.
    \item Extensive experiments across multiple vision-language navigation benchmarks and closed-loop robotic settings demonstrate that SpikingVLA achieves competitive navigation performance with substantially improved inference efficiency. Compared with existing spiking VLA methods, SpikingVLA improves SR and SPL by 11.9\% and 12.6\%, respectively, while asynchronous execution accelerates first-action generation by 11.2$\times$.
\end{itemize}
\section{Related Works}
\textbf{Vision-Language-Action Models:}
VLA models leverage representations learned during vision-language pretraining for robotic control~\citep{brohan2023rt,ma2024survey}. RT-2~\citep{brohan2023rt} casts robot actions as discrete tokens, whereas OpenVLA~\citep{kim2024openvla} introduces an open-source model trained on diverse cross-embodiment data. The paradigm has further been extended to continuous control and navigation~\citep{black2024pi_0,cheng2024navila}. To improve deployment efficiency, OpenVLA-OFT~\citep{kim2025fine} adopts parallel decoding and action chunking, SmolVLA~\citep{shukor2025smolvla} uses a compact architecture, EfficientVLA~\citep{yang2026efficientvla} reduces redundant computation, and QVLA and QuantVLA~\citep{zhang2026quantvla} employ low-bit quantization. Despite these advances, low-latency VLA inference while preserving task performance remains challenging~\citep{liu2024robomamba}.

\textbf{Spiking Neural Networks:}
SNNs process information through discrete spike events, enabling sparse and event-driven computation~\citep{yao2023spikedriven,wang2025spiking,xiao2025rethinking}. Spikformer~\citep{zhou2022spikformer} introduces Spiking Self-attention for visual recognition, while SpikeGPT~\citep{zhu2023spikegpt} extends spiking computation to generative language modeling. To reuse pretrained models, SpikeZIP-TF~\citep{you2024spikeziptf} converts pretrained Transformers into spiking networks without costly retraining. More recent studies further extend spiking computation to multimodal models: SpikeMLLM~\citep{xu2026spikemllm} introduces modality-specific temporal scales, while SpikeVLA applies spike-based processing to VLA tasks. However, existing methods primarily focus on architectural conversion or temporal adaptation, while the heterogeneous activation distributions and temporal imbalance across VLA components remain insufficiently explored.

\section{Preliminary}
\subsection{ANN-to-SNN conversion}
Conventional ANN-to-SNN conversion typically relies on rate coding, where ANN activations are approximated by the temporal firing rates of spiking neurons~\citep{rueckauer2017conversion}. This temporal rate representation often requires a sufficiently long simulation horizon for accurate approximation, while nonlinear operators in Transformer architectures further complicate direct conversion~\citep{jiang2024spatio}. To address this limitation, differential coding~\citep{huang2025differential} represents temporal variations in operator responses instead of repeatedly encoding complete activations.
Specifically, for layer $l$, the differential response and its accumulated representation are defined as follows:
\begin{equation}
x^l[t]
=
t\!\left[
F^l(r^{l-1}[t])-F^l(r^{l-1}[t-1])
\right],
\qquad
r^l[t]
=
r^l[t-1]+\frac{x^l[t]}{t},
\label{eq:differential-coding}
\end{equation}
where $F^l(\cdot)$ denotes the layer operation, $x^l[t]$ is the temporal output increment, and $r^l[t]$ represents the accumulated output at timestep $t$. This formulation enables layer responses to be progressively reconstructed from temporal increments. Since each increment depends only on the current and previous accumulated states, the output can be computed online without requiring future temporal information, enabling efficient timestep-wise propagation~\citep{huang2025differential}.

\subsection{Multi-Threshold Neurons}
\label{MTNeuron}
Differential conversion employs multi-threshold (MT) neurons to encode temporal response increments with multiple signed firing levels~\citep{huang2024towards, huang2025differential}. This multi-level representation improves temporal resolution while preserving event-driven computation~\citep{zhang2025toward, chen2025las}. Specifically, for layer $l$, the neuronal dynamics are defined as follows:
\begin{equation}
    \mathbf{m}^{l}[t]
    =
    \mathbf{v}^{l}[t-1]+\mathbf{I}^{l}[t],
    \qquad
    \mathbf{z}^{l}[t]
    = \operatorname{MT}_{\theta^{l},n}\!\left(\mathbf{m}^{l}[t]\right),
    \qquad
    \mathbf{v}^{l}[t]
    =
    \mathbf{m}^{l}[t]-\mathbf{z}^{l}[t].
    \label{eq:mt_dynamics}
\end{equation}
Here, $\mathbf{m}^{l}[t]$ and $\mathbf{v}^{l}[t]$ denote the membrane potentials before and after spike emission, respectively, $\mathbf{I}^{l}[t]$ represents the incoming differential response, and $\mathbf{z}^{l}[t]$ denotes the emitted signed spike.
The element-wise operator $\operatorname{MT}_{\theta^{l},n}$ employs a symmetric set of firing levels,
$\Lambda^{l}=\{\pm\theta^{l}/2^{k-1}\mid k=1,\ldots,n\}$,
where $\theta^{l}$ denotes the maximum firing amplitude and $n$ specifies the number of levels for each polarity. For membrane potentials satisfying $|\mathbf{m}^{l}[t]|<\theta^{l}/2^{n-1}$, the operator outputs zero; otherwise, the response is assigned to the nearest level in $\Lambda^{l}$.  
For $n=1$, the MT neuron degenerates to a signed IF neuron with thresholds $\pm\theta^{l}$~\citep{hu2024toward}.

\section{Problem Analysis}
\label{sec:problem}
Despite the promising efficiency benefits of ANN-to-SNN conversion, extending existing methods to VLA models remains challenging due to activation imbalance across multimodal components and temporal synchronization overhead during spiking inference. First, VLA components exhibit pronounced channel-wise activation outliers, which increase spike encoding errors at limited timesteps and degrade conversion accuracy. Second, non-uniform temporal horizons across channels and components introduce synchronization stalls, increasing end-to-end inference latency.
\begin{figure}[!htbp]
    \centering
    \includegraphics[width=0.98\linewidth]{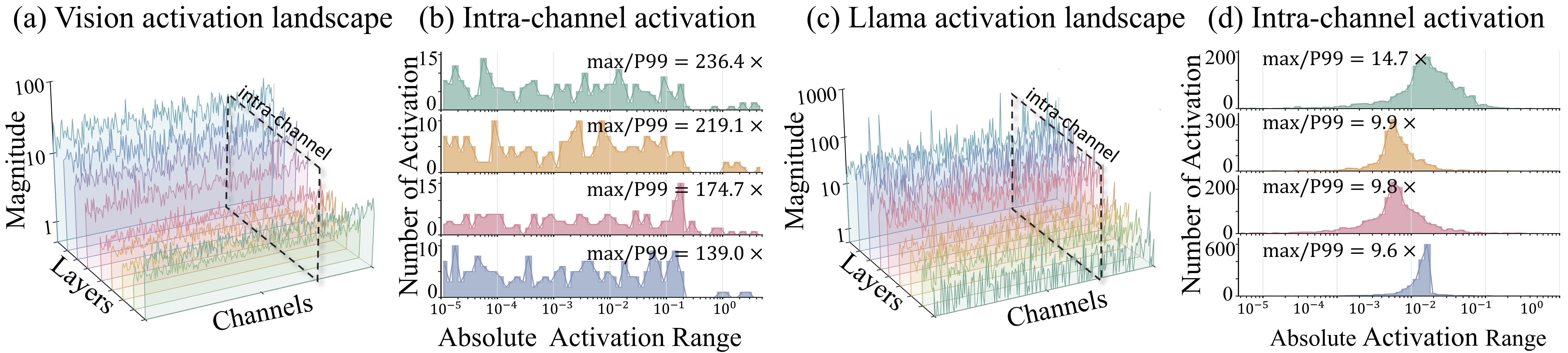}
    \caption{Channel-wise activation imbalance across VLA components.}
    \vspace{-10pt}
    \label{fig:problem}
\end{figure}
\subsection{Channel-Wise Activation Outliers}
VLA models exhibit pronounced channel-wise activation outliers across different components~\citep{zeng2026janusvln}. As shown in Fig.~\ref{fig:problem}(a) and Fig.~\ref{fig:problem}(c), the vision encoder and LLaMA decoder contain a small subset of channels with substantially larger activation magnitudes. These outlier channels dominate the activation scale, increasing spike encoding errors under limited timesteps and making accurate conversion more difficult. Existing ANN-to-SNN conversion methods~\citep{huang2024towardsecmt} typically address such activation imbalance through threshold calibration. Specifically, the threshold of channel $c$ in module $m$ is optimized by minimizing the reconstruction error:
\begin{equation}
\theta_{m,c}^{\star}
=
\arg\min_{\theta>0}
\int
\left[
x-
\frac{\theta}{N}
\operatorname{clip}
\left(
\left\lfloor
\frac{Nx}{\theta}+\frac{1}{2}
\right\rfloor,
0,N
\right)
\right]^2
p_{m,c}(x)\,dx,
\qquad
N\simeq 2^nT,
\end{equation}
where \(p_{m,c}(x)\) denotes the activation distribution of channel \(c\) in module \(m\), \(T\) is the timestep budget, and \(n\) determines the number of firing levels. For a fixed \(T\), a larger \(n\) provides finer spike representation by increasing the available firing levels. However, as shown in Fig.~\ref{fig:problem}(b) and Fig.~\ref{fig:problem}(d), the activation ranges remain strongly long-tailed, with a small subset of channels dominating the overall scale. Increasing the firing resolution therefore does not remove this range imbalance. As a result, extreme channels continue to constrain accurate spike encoding under limited timesteps.

\subsection{Temporal Synchronization Overhead}
VLA models require low response latency for real-time closed-loop interaction~\citep{liu2024robomamba}. Existing spiking VLA methods often assign different temporal horizons across channels and modules to accommodate heterogeneous activation dynamics~\citep{song2026spikevla}. However, these heterogeneous horizons
lead to mismatched completion times across channels. Since each layer can proceed only after all channels complete, its execution time is determined by the channel requiring the longest temporal horizon. Let $T_{m,l,c}$ denote the temporal horizon of channel $c$ in layer $l$ of module $m$.
The resulting layer latency is therefore defined as follows:
\begin{equation}
T_{m,l}^{\mathrm{exec}}=\max_{c}T_{m,l,c},\qquad
\mathcal{W}_{m,l}=\sum_{c=1}^{C_{m,l}}\!\left(T_{m,l}^{\mathrm{exec}}-T_{m,l,c}\right),\qquad
\mathcal{L}_{\mathrm{first}}=\sum_{m\in\mathcal{M}}\sum_{l=1}^{L_m}\tau_{m,l}T_{m,l}^{\mathrm{exec}},
\end{equation}
where $C_{m,l}$ denotes the number of channels, $\mathcal{M}=\{\mathrm{V},\mathrm{L},\mathrm{A}\}$ represents the
vision, language, and action modules, and $\tau_{m,l}$ denotes the execution time per timestep. As illustrated in Fig.~\ref{fig:overview}(b), the latency of each layer is determined by the channel with the longest temporal horizon. Channels that complete earlier remain idle until this horizon is reached, resulting in an accumulated idle budget $\mathcal{W}_{m,l}$. Consequently, shortening the horizons of non-critical channels does not reduce the layer latency unless the maximum horizon is also reduced. These synchronization
delays accumulate across successive layers and modules and ultimately increase the
time-to-first-action (TTFA).

\begin{figure}[!htbp]
    \centering
    \includegraphics[width=0.95\linewidth]{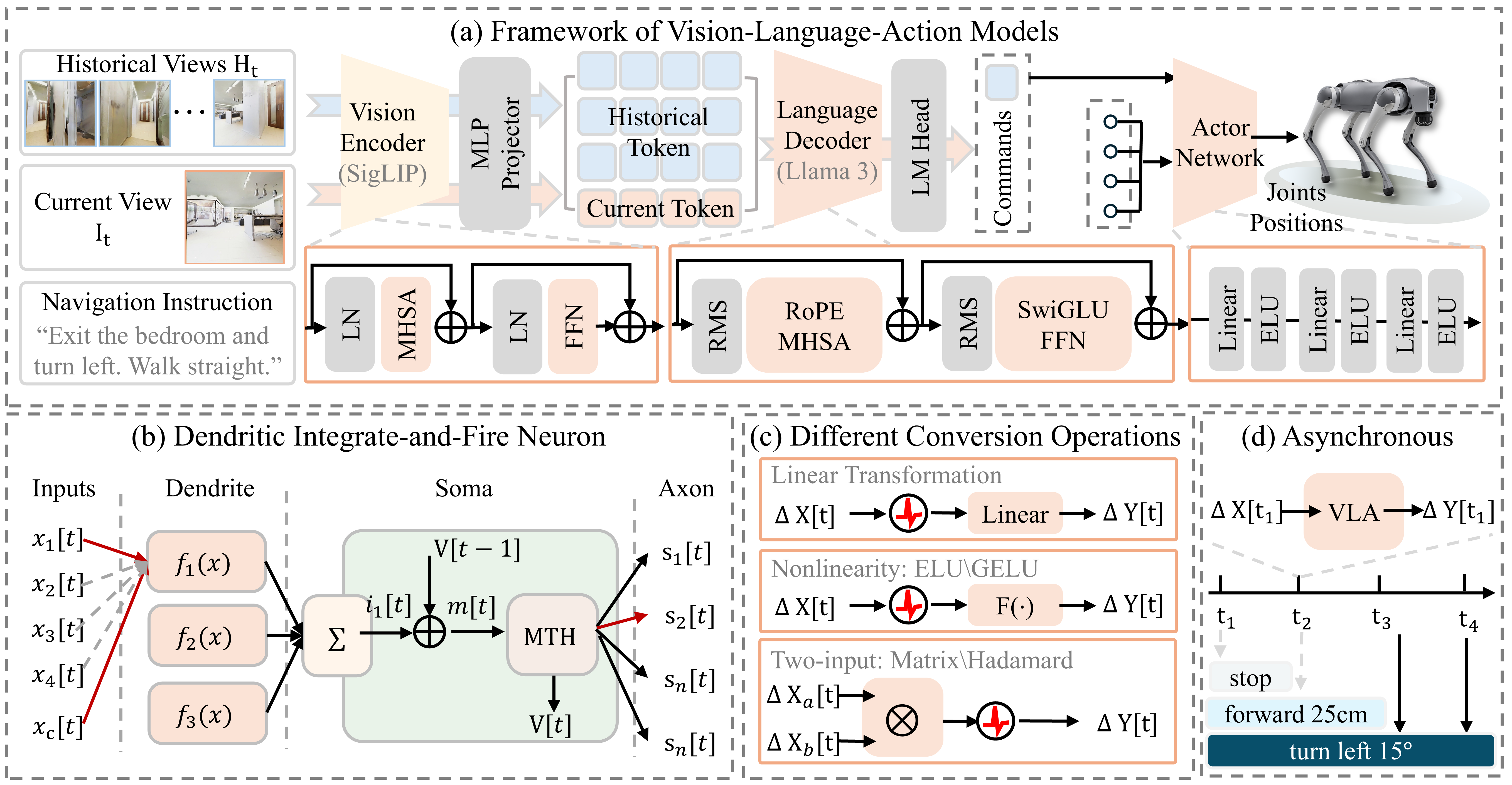}
    \vspace{-10pt}
    \caption{(a) Computation flow of the VLA model. (b) DIF neuron with dendritic feature integration and channel-adaptive spike encoding. (c) Differential conversion strategies for different operators. (d) Event-driven inference by accumulating differential responses over timesteps.}
    \label{fig:methods}
    \vspace{-10pt}
\end{figure}

\section{Methods}
\subsection{Dendritic Integrate-and-Fire Neuron} 
\label{sec:dif} 
To address the representation mismatch induced by activation heterogeneity, we introduce DIF neuron, which couples dendritic channel mixing with channel-adaptive somatic firing. The dendritic compartment redistributes channel responses to reduce inter-channel scale disparity, while the soma adapts its firing range to the residual distribution of each post-dendritic channel.

\subsubsection{Dendritic Channel Mixing}
As analyzed in Sec.~\ref{sec:problem}, VLA activations exhibit substantial channel-wise heterogeneity, with high-magnitude responses concentrated in specific dimensions. We introduce dendritic channel mixing to redistribute these responses before somatic encoding. The dendritic compartment comprises \(D\) parallel branches operating on local channel subspaces with independent temporal states. Given \(\mathbf{x}^{l}[t]\in\mathbb{R}^{C}\), the dendritic dynamics and resulting somatic input are formulated as follows:
\begin{equation}
\mathbf{d}_c^{l}[t+1]
=
\alpha_d^{l}\mathbf{d}_c^{l}[t]
+
(1-\alpha_d^{l})
\mathbf{R}_c^{l}\mathbf{P}_c^{l}\mathbf{x}^{l}[t+1],
\qquad
\mathbf{i}_{s}^{l}[t+1]
=
\sum
(\mathbf{P}_c^{l})^{\top}
\mathbf{d}_d^{l}[t+1].
\label{eq:dendritic_dynamics}
\end{equation}
\(\mathbf{P}_d^{l}\in\{0,1\}^{C_d\times C}\) selects the channel subspace of branch \(d\), \(\mathbf{R}_d^{l}\in\mathbb{R}^{C_d\times C_d}\) performs local channel mixing, and \(\alpha_d^{l}\) controls dendritic state retention. The resulting branch states are mapped back through \((\mathbf{P}_d^{l})^{\top}\) and aggregated at the soma. Owing to the independence of dendritic branches, these transformations can be evaluated in parallel and written in an equivalent layer-wise form as follows:
\begin{equation}
\mathbf{i}_{s}^{l}[t+1]
=
\alpha^{l}\mathbf{i}_{s}^{l}[t]
+
(1-\alpha^{l})
\mathbf{Q}^{l}\mathbf{x}^{l}[t+1],
\qquad
\mathbf{Q}^{l}
=
\sum
(\mathbf{P}_d^{l})^{\top}
\mathbf{R}_d^{l}
\mathbf{P}_d^{l}.
\label{eq:dif_equivalent}
\end{equation}
Here, \(\mathbf{Q}^{l}\) denotes the spatial channel transformation induced by the dendritic branches, while \(\alpha^{l}\) controls temporal integration. When \(\mathbf{Q}^{l}=\mathbf{H}\), the dynamics reduce to Hadamard channel rotation with temporal integration. We further provide theoretical analysis in Appendix~\ref{app:dendritic_analysis}.

\subsubsection{Channel-Adaptive Somatic Firing}
To further improve spike representation after dendritic channel mixing, we introduce channel-adaptive somatic firing.  Specifically, for channel \(c\) in layer \(l\), let \(i_{s,c}^{l}[t]\) denote the aggregated dendritic input at timestep \(t\). The corresponding membrane response and spike output are defined as:
\begin{equation}
m_c^{l}[t]
=
v_c^{l}[t-1]
+
i_{s,c}^{l}[t],
\qquad
z_c^{l}[t]
=
\operatorname{MT}_{\theta_c^{l},n_c^{l}}
\left(
m_c^{l}[t]
\right),
\label{eq:somatic_firing}
\end{equation}
where \(\theta_c^{l}\) controls the channel-specific firing range and \(n\) denotes the fixed number of threshold levels.
To account for the residual distributional variation across post-dendritic channels, we calibrate \((\theta_c^{l}, n_c^{l})\) independently for each channel. Under a fixed timestep budget \(T\), the optimal configuration is obtained by minimizing the channel-wise reconstruction error:
\begin{equation}
\theta_c^{l\star}
=
\arg\min_{\theta>0}
\mathbb{E}_{x\sim\mathcal{D}_{s,c}^{l}}
\left[
\left|
x-\hat{x}_{\theta,n}^{T}
\right|^{2}
\right].
\label{eq:somatic_calibration}
\end{equation}
Here, \(\mathcal{D}_{s,c}^{l}\) denotes the post-dendritic calibration distribution, and \(\hat{x}_{\theta,n}^{T}\) denotes the activation reconstructed from the corresponding \(T\)-step MT response. The calibrated \(\theta_c^{l}\) adapts the encoding range to each post-dendritic channel without increasing the timestep budget. We further visualize the execution dynamics of the proposed DIF neuron. As shown in Fig.~\ref{fig:methods}(b), DIF couples dendritic channel redistribution and temporal integration with channel-adaptive somatic firing.

\subsection{Operation-Aware Differential Conversion with DIF Neurons}
\label{sec:operation}
Building on differential coding~\citep{huangdifferential}, we apply DIF neurons to heterogeneous VLA operators. For the temporal increment \(\Delta\mathbf{Y}^{l}[t]\) of the \(l\)-th operator, DIF performs dendritic channel mixing and temporal integration before somatic firing. Since \(\mathbf{Q}^{l}\) changes the representation basis used for spike encoding, the subsequent operator must account for this transformation. For affine mappings, we exactly absorb \((\mathbf{Q}^{l})^{-1}\) into the input-side weight of the following operator. Accordingly, the DIF encoding and subsequent affine computation are formulated as follows:
\begin{equation}
\mathbf{S}^{l}[t] = \operatorname{DIF}_{\psi^{l}} \!\left( \Delta\mathbf{Y}^{l}[t];\mathbf{Q}^{l} \right), \qquad \widetilde{\mathbf W}^{l+1} = (\mathbf Q^{l})^{-1}\mathbf W^{l+1}, \qquad \Delta\mathbf Z^{l+1}[t] = \mathbf S^{l}[t]\widetilde{\mathbf W}^{l+1}.
\label{eq:dif_reparameterization}
\end{equation}
The reparameterization acts only on the subsequent affine mapping and leaves the DIF dynamics unchanged. For orthogonal \(\mathbf Q^{l}\), the inverse reduces to its transpose and can be fused into the weights offline. Non-affine and multiplicative operators are instead handled by deriving their corresponding differential forms while preserving the original computation. Detail information is as follows.

\subsubsection{Single-Input Operator Conversion}
We first consider operators acting on a single evolving representation, including affine projections and nonlinear transformations. As shown in Fig.~\ref{fig:methods}(c), affine mappings directly propagate the incoming increment through the reparameterized weight, whereas nonlinear operators evaluate the response difference between two consecutive input states. Specifically, it can be defined as follows:
\begin{equation}
\Delta\mathbf{Y}_{\mathcal{F}}[t]
=
\begin{cases}
\Delta\mathbf{X}[t]\widetilde{\mathbf{W}},
&
\mathcal{F}(\mathbf{X})
=
\mathbf{X}\widetilde{\mathbf{W}}+\mathbf{b},
\\
\mathcal{F}\!\left(
\mathbf{X}[t-1]+\Delta\mathbf{X}[t]
\right)
-
\mathcal{F}\!\left(
\mathbf{X}[t-1]
\right),
&
\mathcal{F}\ \text{is non-affine}.
\end{cases}
\label{eq:single_input_conversion}
\end{equation}
For affine mappings, temporal differencing cancels the bias term, reducing the update to a linear mapping of the input increment. For nonlinear operators, including normalization, pointwise activation, and Softmax, the increment is obtained
from consecutive operator responses using the cached state. In both cases, the resulting increment is subsequently processed by DIF.

\subsubsection{Multiplicative Interactions}
We next consider multiplicative interactions between two evolving representations, including attention-score computation, attention-value aggregation, and element-wise gating. As illustrated in Fig.~\ref{fig:methods}(c), for \(\mathcal{F}(\mathbf{A},\mathbf{B})=\mathbf{A}\star\mathbf{B}\), the temporal increment is follows:
\begin{equation}
\Delta\mathbf{Y}_{\mathrm{mul}}[t]
=
\Delta\mathbf{A}[t]\star\mathbf{B}[t-1]
+
\mathbf{A}[t-1]\star\Delta\mathbf{B}[t]
+
\Delta\mathbf{A}[t]\star\Delta\mathbf{B}[t],
\label{eq:multiplicative_conversion}
\end{equation}
where \(\star\) denotes matrix or element-wise multiplication. The first two terms capture the individual operand updates, while the final term accounts for their joint variation. The resulting increment is then encoded by DIF for asynchronous propagation.

\subsection{Overall Architecture}
\label{sec:overall} 
We instantiate SpikingVLA on NaVILA~\citep{cheng2024navila}, a hierarchical vision-language-action policy. As shown in Fig.~\ref{fig:methods}(a), the high-level VLA integrates historical and current visual observations with the navigation instruction through the vision encoder, multimodal projector, and language decoder to produce a spatial navigation command. The low-level actor then combines this command with the current locomotion observations to predict target joint positions.
We apply the proposed DIF-based conversion to both policies while preserving their original computational interfaces. At each control step \(t\), the external observations are fixed over an internal SNN horizon \(\tau\in\{1,\ldots,T\}\); we therefore omit \(t\) from the internal states for clarity. At each internal timestep, the high-level policy accumulates the newly propagated differential response and immediately updates the navigation command, which is subsequently consumed by the low-level policy within the same timestep. The resulting hierarchical execution is formulated as follows:
\begin{equation}
\begin{aligned}
\widehat{\mathbf u}_{h}[\tau]
&=
\widehat{\mathbf u}_{h}[\tau-1]
+
\mathcal C_{\mathrm{DIF}}^{h}
\!\left(\mathbf o_t,\mathbf l\right)[\tau],
\qquad
\mathbf v_{\mathrm{cmd}}[\tau]
=
\mathcal G_{\mathrm{cmd}}
\!\left(\widehat{\mathbf u}_{h}[\tau]\right),
\\
\widehat{\mathbf q}^{d}[\tau]
&=
\widehat{\mathbf q}^{d}[\tau-1]
+
\mathcal C_{\mathrm{DIF}}^{l}
\!\left(
\mathbf o_{l,t},
\mathbf v_{\mathrm{cmd}}[\tau]
\right)[\tau],
\qquad
\tau=1,\ldots,T .
\end{aligned}
\label{eq:overall_architecture}
\end{equation}
\(\mathcal C_{\mathrm{DIF}}^{h}\) and \(\mathcal C_{\mathrm{DIF}}^{l}\) denote full-depth differential execution of the high-level VLA and low-level actor, respectively. At each internal timestep \(\tau\), the high-level state is updated from the current differential response and immediately decoded into \(\mathbf v_{\mathrm{cmd}}[\tau]\), which is consumed by the low-level policy within the same timestep. This produces an updated joint target \(\widehat{\mathbf q}^{d}[\tau]\) at every \(\tau\), with successive differential responses progressively refining the action estimate. As illustrated in Fig.~\ref{fig:methods}(d), executable outputs are therefore available throughout the SNN horizon rather than only at the final timestep.

\section{Experiments}
We evaluate SpikingVLA on three VLN-CE benchmarks: R2R Val-Unseen~\citep{anderson2018vision}, RxR Val-Unseen~\citep{krantz2020beyond}, and VLN-CE-Isaac~\citep{cheng2024navila}. These benchmarks evaluate complementary aspects of embodied navigation, including generalization to unseen environments, long-horizon instruction following, and closed-loop navigation under realistic robot dynamics. We report NE, OS, SR, SPL, and nDTW to evaluate navigation performance. For deployment efficiency, we consider memory footprint, TTFA and Total T~\citep{yao2023spike,yao2025scaling}. 
\begin{figure}[!hbp]
    \centering
    \vspace{-10pt}
    \includegraphics[width=0.98\linewidth]{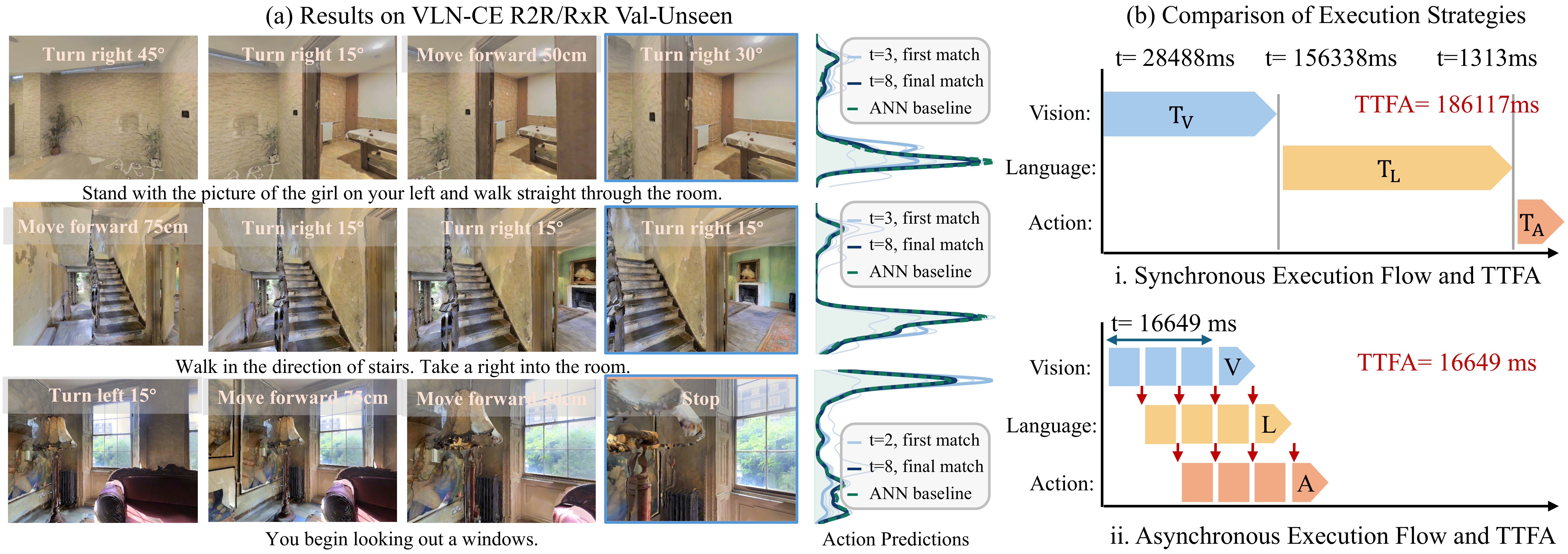}
    \vspace{-10pt}
    \caption{Performance and Efficiency analysis. (a) Navigation trajectories and action-prediction convergence on the VLN-CE R2R and RxR Val-Unseen splits. (b) Inference-latency comparison, highlighting the reduction achieved by asynchronous spike-based execution.}
    \vspace{-15pt}
    \label{fig:exp}
\end{figure}
\subsection{Navigation Performance on VLN-CE Benchmarks}
We evaluate SpikingVLA on the Val-Unseen splits of R2R-CE and RxR-CE against representative VLN and VLA methods~\citep{wei2025streamvln, an2024etpnav}. As shown in Table~\ref{tab:navigation}, SpikingVLA consistently achieves competitive navigation performance across both benchmarks under limited timestep budgets. On R2R-CE Val-Unseen, SpikingVLA achieves an NE of 5.20, OS of 62.4, SR of 54.6, and SPL of 50.3 at T=8. On RxR-CE Val-Unseen, SpikingVLA achieves an NE of 6.30, SR of 51.6, SPL of 45.8, and nDTW of 61.3 at T=8. These results demonstrate that SpikingVLA maintains competitive navigation performance under limited timestep budgets.

Beyond the local simulation budget T, we report Total T to characterize the cumulative end-to-end temporal depth. Under synchronized spiking inference, temporal computation is executed sequentially across layers and VLA components, causing individual temporal horizons to accumulate into an end-to-end temporal depth of approximately 500--600 steps. In contrast, SpikingVLA enables VLA components to process intermediate responses as they become available, allowing temporal execution to overlap and reducing synchronization overhead. To further analyze temporal convergence, Fig.~\ref{fig:exp}(a) presents representative long-horizon trajectories and corresponding action predictions on R2R-CE and RxR-CE under different timestep budgets. As the temporal responses accumulate, spiking action predictions progressively align with the corresponding ANN policy outputs and remain consistent throughout navigation trajectories. These results demonstrate that SpikingVLA preserves ANN-level decision consistency while operating under a compact temporal budget.

\newcolumntype{Y}{>{\centering\arraybackslash}X}
\newcolumntype{C}[1]{>{\centering\arraybackslash}p{#1}}
\newcolumntype{L}[1]{>{\raggedright\arraybackslash}p{#1}}

\begin{table*}[!t]
\centering
\caption{Navigation performance on the R2R-CE Val-Unseen and RxR-CE Val-Unseen. \textbf{TF} indicates training-free paradigms, and \textbf{Total T} denotes the total number of full-network timesteps. \textbf{Bold} values indicate the best results among SNN-based methods.}
\label{tab:navigation}

\begingroup
\small
\renewcommand{\arraystretch}{1.08}
\setlength{\tabcolsep}{2.4pt}

\begin{tabularx}{\linewidth}{
@{}
C{0.55cm}
L{4.15cm}
C{0.48cm}
*{8}{Y}
C{1.05cm}
@{}
}

\toprule
\toprule
\multirow[c]{2}{*}{\textbf{Type}}
&
\multicolumn{1}{c}{
    \multirow[c]{2}{*}{\textbf{Method}}
}
&
\multirow[c]{2}{*}{\textbf{TF}}
&
\multicolumn{4}{c}{\textbf{R2R Val-Unseen}}
&
\multicolumn{4}{c}{\textbf{RxR Val-Unseen}}
&
\multirow[c]{2}{*}{\textbf{Total T}$\downarrow$}
\\

\cmidrule(lr){4-7}
\cmidrule(lr){8-11}

&
&
&
\textbf{NE}$\downarrow$
&
\textbf{OS}$\uparrow$
&
\textbf{SR}$\uparrow$
&
\textbf{SPL}$\uparrow$
&
\textbf{NE}$\downarrow$
&
\textbf{SR}$\uparrow$
&
\textbf{SPL}$\uparrow$
&
\textbf{nDTW}$\uparrow$
&
\\

\midrule

\multirow[c]{5}{*}{
    \rotatebox[origin=c]{90}{\textbf{ANN}}
}
&
StreamVLN~\citep{wei2025streamvln}
&
\ding{55}
&
4.98
&
64.2
&
56.9
&
51.9
&
6.22
&
52.9
&
46.0
&
61.9
&
-
\\

&
InternVLA-N1~\citep{wei2026ground}
&
\ding{55}
&
4.05
&
70.7
&
64.3
&
58.5
&
4.58
&
61.4
&
51.8
&
70.0
&
-
\\

&
ETPNav~\citep{an2024etpnav}
&
\ding{55}
&
4.71
&
65.0
&
57.0
&
49.0
&
5.64
&
54.7
&
44.8
&
61.9
&
-
\\

&
HNR~\citep{wang2024lookahead}
&
\ding{55}
&
4.42
&
67.0
&
61.0
&
51.0
&
5.50
&
56.3
&
46.7
&
63.5
&
-
\\

\cmidrule(lr){2-12}

&
NaVILA~\citep{cheng2024navila}
&
\ding{55}
&
5.28
&
61.5
&
53.9
&
49.3
&
6.12
&
52.3
&
46.1
&
61.0
&
-
\\

\midrule

\multirow[c]{2}{*}{
    \rotatebox[origin=c]{90}{\textbf{SNN}}
}
&
SpikeVLA~\citep{song2026spikevla}
&
\ding{55}
&
5.38
&
\textbf{63.4}
&
53.3
&
47.9
&
\textbf{6.20}
&
\textbf{51.9}
&
45.3
&
60.4
&
\mbox{500-600}
\\

\cmidrule(lr){2-12}

&
\textbf{SpikingVLA (Ours)}
&
\cellcolor{blue!6}\ding{51}
&
\cellcolor{blue!6}\textbf{5.20}
&
\cellcolor{blue!6}62.4
&
\cellcolor{blue!6}\textbf{54.6}
&
\cellcolor{blue!6}\textbf{50.3}
&
\cellcolor{blue!6}6.30
&
\cellcolor{blue!6}51.6
&
\cellcolor{blue!6}\textbf{45.8}
&
\cellcolor{blue!6}\textbf{61.3}
&
\cellcolor{blue!6}\textbf{8}
\\
\bottomrule
\bottomrule

\end{tabularx}

\endgroup
\vspace{-10pt}
\end{table*}

\newcolumntype{Y}{>{\centering\arraybackslash}X}
\newcolumntype{C}[1]{>{\centering\arraybackslash}p{#1}}

\begin{table*}[!t]
\centering
\caption{
Closed-loop navigation performance and deployment efficiency on VLN-CE-Isaac.}
\label{tab:isaac_efficiency}
\begingroup
\small
\renewcommand{\arraystretch}{1.06}
\setlength{\tabcolsep}{3.0pt}
\begin{tabularx}{\linewidth}{
@{}
c
>{\raggedright\arraybackslash}X
C{0.45cm}
*{4}{C{0.8cm}}
C{1.0cm}
C{1.0cm}
C{1.0cm}
@{}
}

\toprule
\toprule
\multirow[c]{2}{*}{\textbf{Type}}
&
\multicolumn{1}{c}{
    \multirow[c]{2}{*}{\textbf{Method}}
}
&
\multirow[c]{2}{*}{\textbf{TF}}
&
\multicolumn{4}{c}{\textbf{Closed-Loop Navigation}}
&
\multicolumn{3}{c}{\textbf{Deployment Efficiency}}
\\

\cmidrule(lr){4-7}
\cmidrule(lr){8-10}

&
&
&
\textbf{NE}$\downarrow$
&
\textbf{OS}$\uparrow$
&
\textbf{SR}$\uparrow$
&
\textbf{SPL}$\uparrow$
&
\textbf{Mem.}$\downarrow$
&
\textbf{TTFA}$\downarrow$
& 
\mbox{\textbf{Total T}$\downarrow$}
\\

\midrule

\multirow[c]{4}{*}{
    \rotatebox[origin=c]{90}{\textbf{ANN}}
}

& NaVILA-Blind~\citep{cheng2024navila}
& \ding{55}
& 6.03
& 49.0
& 36.2
& 33.3
& 16126.2
& -
& -

\\

& NaVILA-Vision~\citep{cheng2024navila}
& \ding{55}
& 5.49
& 58.7
& 50.2
& 45.5
& 16126.2
& -
& -

\\

& NaVILA-R~\citep{song2026spikevla}
& \ding{55}
& 6.29
& 52.1
& 36.5
& 29.5
& 16126.2
& -
& -

\\

& NaVILA-AWQ~\citep{cheng2024navila}
& \ding{55}
& 6.58
& 48.2
& 32.8
& 27.4
& 5846.9
& -
& -

\\

\midrule

\multirow[c]{3}{*}{
    \rotatebox[origin=c]{90}{\textbf{SNN}}
}

& SpikeVLA~\citep{song2026spikevla}
& \ding{55}
& 6.02
& 53.6
& 32.7
& 28.5
& \textbf{6251.5}
& 186.1
& \mbox{500-600}

\\

\cmidrule(lr){2-10}

& \textbf{SpikingVLA-Blind (Ours)}
& \cellcolor{blue!6}\ding{51}
& \cellcolor{blue!6}6.11
& \cellcolor{blue!6}48.1
& \cellcolor{blue!6}34.4
& \cellcolor{blue!6}31.1
& \cellcolor{blue!6}\textbf{6251.5}
& \cellcolor{blue!6}\textbf{16.7}
& \cellcolor{blue!6}\textbf{8}
\\

& \textbf{SpikingVLA-Vision (Ours)}
& \cellcolor{blue!6}\ding{51}
& \cellcolor{blue!6}\textbf{5.45}
& \cellcolor{blue!6}\textbf{56.6}
& \cellcolor{blue!6}\textbf{44.6}
& \cellcolor{blue!6}\textbf{41.1}
& \cellcolor{blue!6}\textbf{6251.5}
& \cellcolor{blue!6}\textbf{16.7}
& \cellcolor{blue!6}\textbf{8}
\\

\bottomrule
\bottomrule

\end{tabularx}

\endgroup
\vspace{-10pt}
\end{table*}

\subsection{Closed-loop Embodied Navigation}
We further evaluate SpikingVLA on VLN-CE-Isaac using the Unitree Go2 embodiment. As shown in Table~\ref{tab:isaac_efficiency}, both SpikingVLA-Blind and SpikingVLA-Vision preserve strong navigation performanc. At $T=8$, SpikingVLA-Blind achieves an NE/OS/SR/SPL of $6.11/48.1/34.4/31.1$, while SpikingVLA-Vision further improves the results to $5.45/56.6/44.6/41.1$ by leveraging visual observations. Compared with its ANN counterpart, NaVILA-Vision, SpikingVLA-Vision maintains comparable navigation performance, with only 5.6 and 4.4 percentage-point gaps in SR and SPL, respectively. Moreover, compared with an existing spiking VLA baseline, SpikingVLA-Vision improves SR and SPL by 11.9 and 12.6 percentage points, demonstrating the effectiveness of the proposed conversion framework for closed-loop embodied navigation.

SpikingVLA further improves deployment efficiency through asynchronous execution. As illustrated in Fig.~\ref{fig:exp}(b), synchronized VLA inference serializes vision, language, and action processing, forcing each downstream component to wait for upstream temporal computation to complete. These synchronization stalls accumulate along the VLA pipeline and substantially increase first-action latency. In contrast, SpikingVLA propagates intermediate spike responses as they become available, allowing downstream components to start computation earlier and enabling temporal overlap across vision, language, and action processing. As a result, TTFA decreases from 186.1 to 16.7, corresponding to an approximately $11.1\times$ speedup in first-action generation. These results demonstrate that asynchronous spike propagation effectively enables low-latency closed-loop VLA inference.

\newpage
\begin{wrapfigure}{r}{0.50\linewidth}
    \centering
    \vspace{-5pt}
    \includegraphics[width=0.98\linewidth]{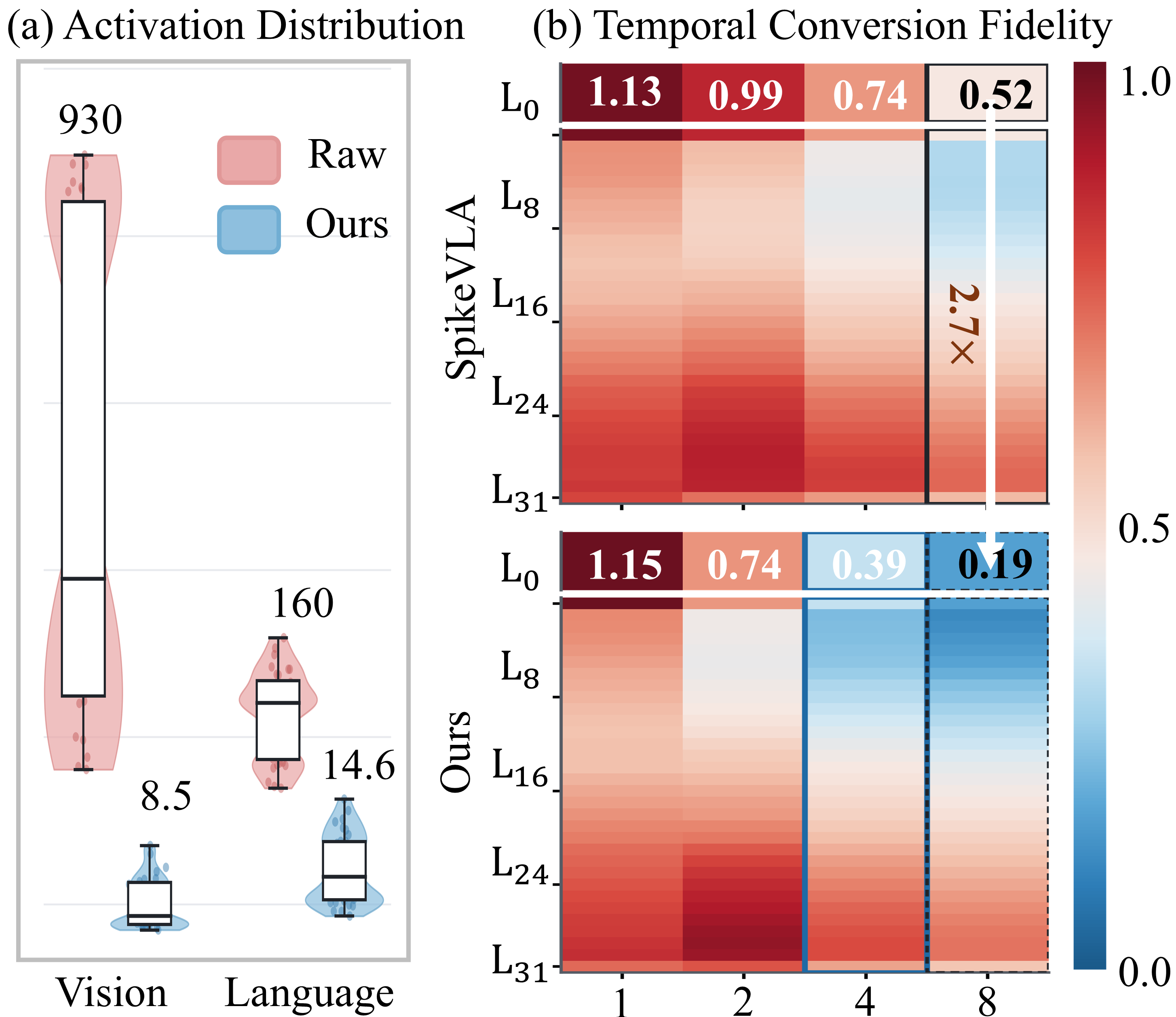}
    \caption{(a) Channel-wise activation distributions. (b) Layer-wise conversion discrepancy under different simulation timesteps. 
    }
    \label{fig:dif_distribution}
    \vspace{-5pt}
\end{wrapfigure}

\subsection{Quantization Analysis}
To investigate whether DIF alleviates activation imbalance in VLA models, we visualize channel-wise activations before and after dendritic mixing in Fig.~\ref{fig:dif_distribution}(a). DIF produces more balanced somatic inputs in both the vision and LLaMA blocks by reducing the dominance of extreme channels. For representative channels, the activation ranges decrease from 15.1 to 7.1 and from 30.2 to 7.0, respectively. These results indicate that dendritic mixing effectively balances channel responses and stabilizes activation distributions for spike encoding.

We further evaluate conversion fidelity by measuring ANN-SNN representation discrepancy in the vision encoder as a representative VLA component under different timestep budgets. As shown in Fig.~\ref{fig:dif_distribution}(b), DIF consistently achieves lower conversion errors than the SpikeVLA baseline under the same conversion settings~\citep{song2026spikevla}. At $T=8$, DIF reduces the representation discrepancy from 0.52 to 0.19. This improvement demonstrates that DIF preserves ANN representations more effectively under limited temporal budgets, reducing information loss caused by insufficient timesteps.


\begin{wraptable}{r}{0.52\linewidth}
\vspace{-15pt}
\centering
\caption{
Module-wise ANN--SNN conversion error under different spiking neuron configurations.
}
\label{tab:ablation}

\begingroup
\small
\renewcommand{\arraystretch}{1.05}
\setlength{\tabcolsep}{3.8pt}

\begin{tabular}{lccccc}
\toprule
\multirow{2}{*}{\textbf{Module}}
&
\multirow{2}{*}{\textbf{Neuron}}
&
\multicolumn{4}{c}{\textbf{Number of Thresholds}}
\\
\cmidrule(lr){3-6}

&
&
$\mathbf{n=1}$
&
$\mathbf{n=2}$
&
$\mathbf{n=4}$
&
$\mathbf{n=8}$
\\

\midrule

\multirow{2}{*}{\makecell[c]{Vision\\Encoder}}
& MT
& 1.13
& 0.99
& 0.74
& 0.52
\\
& DIF
& 1.15
& 0.74
& 0.39
& 0.11
\\

\midrule

\multirow{2}{*}{\makecell[c]{LLaMA\\Decoder}}
& MT
& 0.48
& 0.40
& 0.31
& 0.24
\\
& DIF
& 0.52
& 0.36
& 0.22
& 0.11
\\

\bottomrule
\end{tabular}

\endgroup
\vspace{-10pt}
\end{wraptable}

\subsection{Ablation Study}
We further evaluate module-wise conversion fidelity under different spiking neuron configurations~\citep{huang2024towardsecmt}. As shown in Table~\ref{tab:ablation}, increasing the number of thresholds consistently reduces the ANN--SNN representation error in both the vision encoder and LLaMA decoder. Notably, DIF achieves progressively larger improvements over MT as the firing resolution increases. In the vision encoder, DIF reduces the conversion error from \(0.99\) to \(0.74\), \(0.74\) to \(0.39\), and \(0.52\) to \(0.11\) for \(n=2,4,8\), corresponding to relative reductions of \(25.3\%\), \(47.3\%\), and \(78.8\%\), respectively. Similar trends are observed in the LLaMA decoder, where the error decreases from \(0.40\) to \(0.36\), \(0.31\) to \(0.22\), and \(0.24\) to \(0.11\), yielding reductions of \(10.0\%\), \(29.0\%\), and \(54.2\%\). These consistent gains across heterogeneous visual and language representations validate the effectiveness of dendritic channel redistribution for improving ANN--SNN conversion fidelity.

\section{Conclusion}
We present SpikingVLA, an ANN-to-SNN conversion framework that enables accurate and low-latency spiking inference for pretrained vision-language-action models. SpikingVLA addresses two fundamental challenges in spiking VLA conversion: channel-wise activation outliers that limit conversion fidelity under constrained timesteps and temporal synchronization bottlenecks that hinder efficient inference. To overcome these challenges, SpikingVLA introduces a Dendritic Integrate-and-Fire (DIF) neuron to improve spike representation fidelity and an asynchronous execution mechanism to enable temporal overlap across VLA components. Extensive experiments demonstrate that SpikingVLA preserves competitive navigation performance while substantially reducing temporal execution overhead. Further analyses verify that DIF improves ANN--SNN conversion fidelity across both visual and language components. These results highlight SpikingVLA as an effective framework for scalable and efficient event-driven VLA inference.

\section{Use of Large Language Model}
In this manuscript, we used a LLM solely for language refinement and manuscript polishing. Specifically, the LLM was used to improve sentence clarity, correct grammatical errors, and enhance the overall readability and academic style of the manuscript. The LLM was not involved in the development of research ideas, methodology design, experimental implementation, result analysis, or generation of scientific claims. All technical contributions, experimental results, and conclusions presented in this paper were independently developed and verified by the authors.

\bibliography{iclr2027_conference}
\bibliographystyle{iclr2027_conference}
\newpage

\appendix

\section{Navigation Benchmarks and Evaluation Metrics}
\label{app:benchmarks}
We evaluate SpikingVLA across three complementary navigation benchmarks covering continuous visual-language navigation and physics-based robotic execution. Accordingly, we report standard navigation metrics to measure task performance and deployment metrics to characterize latency, memory, and energy efficiency. Detail information is as follows.

\subsection{Navigation Benchmarks}
\textbf{R2R-CE: }R2R-CE extends the Room-to-Room (R2R) benchmark~\citep{anderson2018vision} from discrete navigation graphs to continuous Matterport3D environments in Habitat~\citep{chang2017matterport3d}. Each episode specifies a natural-language instruction, an initial pose, a goal location, and a continuous reference trajectory. We evaluate on the Val-Unseen split, which contains 1,839 episodes across 11 environments disjoint from the training scenes. R2R-CE therefore primarily evaluates generalization to unseen indoor environments under continuous navigation.

\textbf{RxR-CE: }RxR-CE extends Room-Across-Room (RxR)~\citep{ku2020room} to the continuous VLN-CE setting. Compared with R2R-CE, RxR-CE contains longer and more detailed navigation instructions and trajectories, placing greater emphasis on long-horizon instruction following and trajectory fidelity~\citep{krantz2020beyond}. RxR additionally provides multilingual annotations in English, Hindi, and Telugu, together with guide and follower trajectories. Following the standard English VLN-CE setting, we evaluate the Val-Unseen guide split using the English annotations.

\textbf{VLN-CE-Isaac: }VLN-CE-Isaac~\citep{cheng2024navila} extends continuous VLN evaluation from high-level navigation to physics-based robotic execution. The benchmark transfers R2R Val-Unseen environments to Isaac Sim and selects 1,077 physically traversable trajectories from the original 1,839 episodes using high-quality scene meshes. Unlike Habitat-based R2R-CE and RxR-CE, VLN-CE-Isaac explicitly models robot embodiment, joint-level motion, collision, and low-level locomotion dynamics. We use the Unitree Go2 embodiment and evaluate the complete hierarchical pipeline from language-conditioned navigation to joint-level locomotion control.

\subsection{Evaluation Metrics}
We evaluate navigation performance using standard VLN metrics~\citep{cheng2024navila}, including Navigation Error (NE), Oracle Success (OS), Success Rate (SR), Success weighted by Path Length (SPL), and Normalized Dynamic Time Warping (nDTW). NE evaluates the geodesic distance between the final agent position and the goal, where lower values indicate better endpoint accuracy. OS and SR quantify goal-reaching success, with OS considering whether the agent enters the goal region at any point during navigation and SR measuring whether the final position reaches the goal region under the standard 3\,m radius defined by the VLN-CE protocol. SPL further accounts for path efficiency by penalizing successful trajectories that deviate from the shortest path. nDTW evaluates trajectory-level alignment between predicted and reference trajectories.

\textbf{Deployment Metrics.}
For VLN-CE-Isaac, we report time-to-first-action (TTFA), peak memory footprint, and estimated energy consumption to characterize deployment efficiency. TTFA measures the wall-clock latency from receiving an observation to producing the first executable action: 
\begin{equation}
\mathrm{TTFA}
=
t_{\mathrm{first\ action}}
-
t_{\mathrm{input}}.
\label{eq:ttfa_metric}
\end{equation}
Memory footprint records the maximum runtime memory usage during inference, while estimated energy consumption measures the computational energy required for action generation. All deployment metrics are evaluated under the same execution settings for a controlled comparison.

\begin{table}[t]
\centering
\caption{
Module-wise theoretical energy consumption breakdown of VLA inference.
Energy is estimated using the MAC/AC-based model under a 45-nm CMOS technology.
}
\label{tab:energy_breakdown}

\small
\setlength{\tabcolsep}{5pt}

\begin{tabular}{lccccc}
\toprule
\textbf{Method}
&
\textbf{Vision Encoder}
&
\textbf{Projector}
&
\textbf{Language Decoder}
&
\textbf{Actor Network}
&
\textbf{Total}
\\
\midrule

NaVILA-Vision (ANN)
&
24.53
&
0.51
&
120.33
&
2.90
&
148.27
\\

SpikingVLA ($T=4$)
&
3.87
&
0.31
&
19.05
&
0.45
&
23.68
\\

SpikingVLA ($T=8$)
&
7.70
&
0.35
&
37.84
&
0.91
&
46.80
\\

\bottomrule
\end{tabular}
\end{table}

\subsection{Theoretical Energy Consumption Analysis}
We estimate the theoretical inference energy of SpikingVLA using a standard operation-based energy model for SNNs~\citep{yao2025scaling, zhang2026neural, wang2026bipolar}. The total energy is estimated from dense multiply--accumulate (MAC) and spike-driven accumulate (AC) operations:
\begin{equation}
E^{(i)}
=
E_{\mathrm{MAC}}N_{\mathrm{MAC}}^{(i)}
+
E_{\mathrm{AC}}
\sum_{l\in\mathcal{L}_{\mathrm{spk}}}
N_l^{\mathrm{syn}}
\sum_{\tau=1}^{T_i^{\mathrm{eff}}}
r_l^{(i)}[\tau],
\label{eq:energy}
\end{equation}
where $N_{\mathrm{MAC}}^{(i)}$ denotes the number of dense MAC operations for input $i$, $N_l^{\mathrm{syn}}$ represents the synaptic operation count of spiking layer $l$, $r_l^{(i)}[\tau]$ denotes the average firing activity at timestep $\tau$, and $T_i^{\mathrm{eff}}$ represents the effective number of executed timesteps. Following previous studies, we adopt a 45-nm CMOS energy model with $E_{\mathrm{MAC}}=4.6$ pJ and $E_{\mathrm{AC}}=0.9$ pJ~\citep{horowitz20141}. The additional computation introduced by dendritic mixing is accounted for in the AC operations.

Table~\ref{tab:energy_breakdown} summarizes the module-wise theoretical energy consumption of ANN and SpikingVLA inference. Compared with dense ANN execution~\citep{cheng2024navila}, SpikingVLA substantially reduces energy consumption across all VLA components through sparse spike-driven computation. The LLaMA-3 language decoder dominates the overall energy cost due to its large computational scale, while SpikingVLA reduces its energy consumption from 120.33 J to 19.05 J and 37.84 J at $T=4$ and $T=8$, respectively. Similar energy reductions are observed in the vision encoder and actor network, demonstrating the effectiveness of spike-based execution across the VLA pipeline. Increasing the timestep budget from $T=4$ to $T=8$ improves navigation performance with additional energy consumption, while maintaining substantially lower energy than dense ANN inference.

\section{Detail of DIF Neurons}
\label{app:dendritic_analysis}
We analyze the proposed Dendritic Integrate-and-Fire (DIF) neuron from three aspects: theoretical properties, calibration behavior, and activation redistribution. Theoretical analysis characterizes the role of dendritic mixing and adaptive somatic firing, calibration experiments evaluate conversion fidelity under different settings, and activation visualization illustrates the mitigation of channel-wise activation imbalance. Detailed results are provided below.

\subsection{Theoretical Properties}
We analyze the effect of dendritic mixing on channel-wise activation concentration in VLA representations. Using the formulation in
Sec.~\ref{sec:dif}, the effective dendritic transformation is follows:
\begin{equation}
\mathbf Q
=
\sum_{d=1}^{D}
\mathbf P_d^{\top}\mathbf R_d\mathbf P_d,
\qquad
\mathbf x_d
=
\mathbf P_d\mathbf x,
\qquad
\mu_d
=
\sqrt{C_d}\,
\|\mathbf R_d\|_{\max}.
\label{eq:app_dendritic_transform}
\end{equation}
Here, $\mathbf x_d$ denotes the representation assigned to branch $d$, the selectors $\{\mathbf P_d\}_{d=1}^{D}$ form a disjoint partition of the channel space, and $\mathbf R_d^{\top}\mathbf R_d=\mathbf I_{C_d}$. The quantity $\mu_d$ characterizes the coherence of the local mixing matrix, with $\|\mathbf R_d\|_{\max} = \max_{i,j}|[\mathbf R_d]_{ij}|$. A lower coherence corresponds to a more diffuse redistribution across the branch channels.

\begin{theorem}[Energy-Preserving Outlier Redistribution]
\label{thm:outlier_redistribution}
Let $\mathbf o_d\in\mathbb R^{C_d}$ denote a nonzero outlier component within branch $d$, supported on at most $k_d$ channels.
Under the disjoint branch partition and orthogonal local transformations, dendritic mixing satisfies
\begin{equation}
\|\mathbf Q\mathbf x\|_2
=
\|\mathbf x\|_2,
\qquad
\|\mathbf R_d\mathbf o_d\|_\infty
\le
\mu_d
\sqrt{\frac{k_d}{C_d}}
\|\mathbf o_d\|_2 .
\label{eq:app_outlier_redistribution}
\end{equation}
The first result establishes norm preservation, whereas the second bounds the peak concentration of a channel-sparse outlier component. Moreover, if $\mu_d k_d<\sqrt{C_d}$, then $\|\mathbf R_d\mathbf o_d\|_\infty < \|\mathbf o_d\|_\infty$. Thus, sufficiently sparse outliers are reduced in peak magnitude through low-coherence redistribution without suppressing their $\ell_2$ energy.
\end{theorem}

\begin{proof}
Since the branch selectors form a disjoint partition, all cross-branch terms vanish. Together with the orthogonality of each local
transformation, this gives
\begin{equation}
\mathbf Q^{\top}\mathbf Q
=
\sum_{d=1}^{D}
\mathbf P_d^{\top}
\mathbf R_d^{\top}\mathbf R_d
\mathbf P_d
=
\sum_{d=1}^{D}
\mathbf P_d^{\top}\mathbf P_d
=
\mathbf I_C .
\label{eq:app_q_orthogonality}
\end{equation}
Hence, $\mathbf Q$ is orthogonal and $\|\mathbf Q\mathbf x\|_2=\|\mathbf x\|_2$. For the outlier component, its $k_d$-sparse support gives $\|\mathbf o_d\|_1 \le \sqrt{k_d}\|\mathbf o_d\|_2$. Using the definition of $\mu_d$, we obtain
\begin{equation}
\|\mathbf R_d\mathbf o_d\|_\infty
\le
\|\mathbf R_d\|_{\max}\|\mathbf o_d\|_1
\le
\frac{\mu_d}{\sqrt{C_d}}
\sqrt{k_d}\|\mathbf o_d\|_2
=
\mu_d
\sqrt{\frac{k_d}{C_d}}
\|\mathbf o_d\|_2 .
\label{eq:app_outlier_bound_proof}
\end{equation}
Furthermore, $\|\mathbf o_d\|_2 \le \sqrt{k_d}\|\mathbf o_d\|_\infty$, and therefore $\|\mathbf R_d\mathbf o_d\|_\infty \le (\mu_d k_d/\sqrt{C_d}) \|\mathbf o_d\|_\infty$. The sufficient condition $\mu_d k_d<\sqrt{C_d}$ then yields strict peak reduction.
\end{proof}

Theorem~\ref{thm:outlier_redistribution} shows that dendritic mixing reduces channel-wise concentration through redistribution rather than clipping or direct attenuation. Orthogonality preserves the representation norm, while low coherence limits the extent to which a sparse outlier response remains concentrated in individual channels. For a normalized Hadamard transformation, $\mu_d=1$; in the single-channel case ($k_d=1$), the peak magnitude is reduced by a factor of $1/\sqrt{C_d}$ while the $\ell_2$ norm remains unchanged.

The preceding result characterizes the spatial redistribution induced by the dendritic branches. The somatic firing range is calibrated from the post-mixing activation statistics. Therefore, reduced peak concentration can yield a smaller effective firing range. We denote this calibrated reduction by $\theta_{\mathrm{DIF}} = \beta\theta_{\mathrm{base}}$, where $0<\beta<1$, and analyze its implication for temporal spike encoding below.

\begin{theorem}[Temporal--Threshold Granularity Trade-off]
\label{thm:temporal_threshold_tradeoff}
Consider a single-emission MT neuron with firing amplitudes $\{0,\pm\theta/2^{k-1}\}_{k=1}^{n}$ over $T$ timesteps. Under uniform temporal accumulation, the decoded response lies on a nominal lattice with spacing $\theta/(T2^{n-1})$. Suppose calibration after dendritic redistribution yields $\theta_{\mathrm{DIF}} = \beta\theta_{\mathrm{base}}$, with $0<\beta<1$. Then the DIF configuration attains a nominal granularity no coarser than that of the baseline whenever
\begin{equation}
\frac{T_{\mathrm{DIF}}}{T_{\mathrm{base}}}
\ge
\beta\,
2^{\,n_{\mathrm{base}}-n_{\mathrm{DIF}}}.
\label{eq:app_temporal_threshold_tradeoff}
\end{equation}
For a fixed firing-level budget, a smaller calibrated range permits the same nominal granularity to be attained with a shorter temporal horizon. Conversely, increasing the number of firing levels can compensate for a more constrained temporal budget.
\end{theorem}

\begin{proof}
Let $\delta(T,n,\theta) = \theta/(T2^{n-1})$ denote the nominal lattice spacing. Since every admissible firing amplitude is an integer multiple of $\theta/2^{n-1}$, uniform accumulation over $T$ timesteps produces a decoded response on the lattice
$\delta(T,n,\theta)\mathbb Z$. To match the baseline nominal granularity, the DIF configuration must satisfy
\begin{equation}
\frac{\beta\theta_{\mathrm{base}}}
{T_{\mathrm{DIF}}2^{n_{\mathrm{DIF}}-1}}
\le
\frac{\theta_{\mathrm{base}}}
{T_{\mathrm{base}}2^{n_{\mathrm{base}}-1}}
\quad\Longleftrightarrow\quad
\frac{T_{\mathrm{DIF}}}{T_{\mathrm{base}}}
\ge
\beta\,
2^{\,n_{\mathrm{base}}-n_{\mathrm{DIF}}}.
\label{eq:app_tradeoff_proof}
\end{equation}
This relation explicitly characterizes the interaction between temporal depth and firing-level granularity. When  $n_{\mathrm{DIF}}=n_{\mathrm{base}}$, the temporal requirement scales with the calibrated range-reduction factor $\beta$. Increasing
$n_{\mathrm{DIF}}$ further relaxes this requirement, allowing a shorter temporal horizon while preserving the same nominal granularity. This proves Eq.~\ref{eq:app_temporal_threshold_tradeoff}.
\end{proof}
Together, Theorems~\ref{thm:outlier_redistribution} and~\ref{thm:temporal_threshold_tradeoff} provide a mechanism-level explanation for the low-timestep effectiveness of DIF. Dendritic mixing reduces channel-wise activation concentration while preserving representation energy, and the resulting compact firing range improves the temporal--threshold trade-off of the somatic encoder. The corresponding gains in ANN--SNN conversion fidelity are further supported by the empirical analyses below.
The resulting improvements in ANN--SNN conversion fidelity are further validated empirically in the following analyses.
\subsection{Calibration Analysis}
\label{app:calibration_analysis}
We further evaluate the effect of distribution-aware calibration on multimodal ANN--SNN conversion. The pretrained ANN parameters remain fixed, while neuronal parameters are calibrated using the activation statistics of each converted module. All variants share the same conversion architecture and temporal budget. We compare uncalibrated conversion, conventional MT calibration, and DIF calibration on the vision encoder and LLaMA decoder.
\begin{table}[htbp]
\centering
\caption{
Effect of calibration on ANN--SNN representation fidelity. Module-wise representation discrepancy is reported under different
simulation timesteps $T$. Lower is better.
}
\label{tab:app_calibration}

\setlength{\tabcolsep}{4.8pt}
\renewcommand{\arraystretch}{1.08}

\begin{tabular}{c|ccc|ccc}
\toprule
& \multicolumn{3}{c|}{Vision Encoder}
& \multicolumn{3}{c}{LLaMA Decoder} \\
\cmidrule(lr){2-4}
\cmidrule(lr){5-7}
$T$
& Uncalib.
& MT-Calib.
& DIF-Calib.
& Uncalib.
& MT-Calib.
& DIF-Calib. \\
\midrule
1 & 1.00 & 1.13 & 1.15
  & 1.30 & 0.48 & 0.52 \\
2 & 1.02 & 0.99 & 0.74
  & 1.31 & 0.40 & 0.36 \\
4 & 1.07 & 0.74 & 0.39
  & 1.26 & 0.31 & 0.22 \\
8 & 1.12  & 0.52 & 0.19
  & 1.24 & 0.24 & 0.11 \\
\bottomrule
\end{tabular}
\end{table}

As shown in Table~\ref{tab:app_calibration}, calibration consistently reduces the ANN--SNN representation discrepancy in both visual and language modules, with DIF providing further improvements over conventional MT calibration. Dendritic redistribution reduces
outlier-induced activation concentration before somatic calibration, allowing the firing range to better match the dominant activation distribution. The resulting gain is particularly evident under limited temporal budgets, demonstrating improved finite-level spike encoding with fewer timesteps.

\subsection{Visualization of Activation Redistribution}
\label{app:activation_visualization}
We further visualize the layer-wise activation distributions to examine the redistribution induced by dendritic mixing. As shown in Fig.~\ref{fig:app_activation_redistribution}, the MLP down-projection inputs exhibit pronounced long-tailed distributions across the LLaMA decoder, with a small subset of channels dominating the activation range. After dendritic mixing, these concentrated responses are redistributed across channels, resulting in consistently more compact activation distributions over network depth.
\begin{figure}[!htbp]
    \centering
    \includegraphics[width=0.98\linewidth]{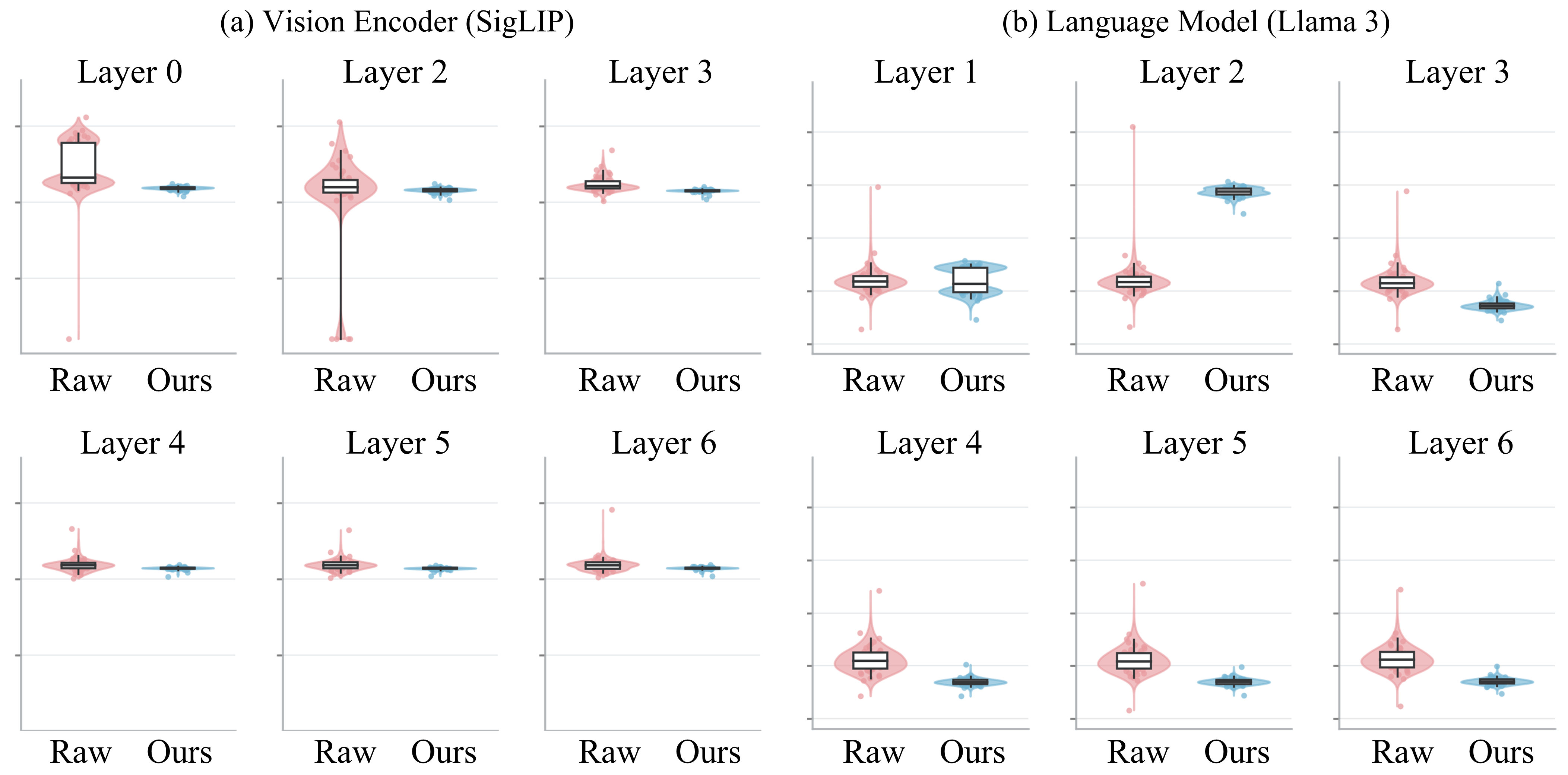}
    \vspace{-10pt}
    \caption{Activation redistribution across LLaMA decoder layers.}
    \label{fig:app_activation_redistribution}
\end{figure}

The layer-wise trend is consistent with Theorem~\ref{thm:outlier_redistribution}: dendritic mixing reduces the concentration of outlier-dominated channel responses while preserving the representation norm. The resulting dynamic range is therefore better
matched to finite-level somatic firing, providing more favorable conditions for accurate spike encoding under constrained temporal and firing-level budgets.

\newpage
\section{Component-Wise Differential Implementation}
\label{app:component_impl}
To clarify the implementation of operation-aware differential conversion, we detail the execution of the major components in Fig.~\ref{fig:methods}(a). We retain the original VLA operators and instantiate differential propagation according to each operator structure. Affine operators propagate incoming increments through reparameterized weights, nonlinear operators evaluate temporal differences between consecutive responses, and multiplicative operators follow the exact bilinear differential expansion. The resulting increments are then encoded by DIF neurons and propagated to subsequent operators.

\begin{figure}
    \centering
    \includegraphics[width=0.98\linewidth]{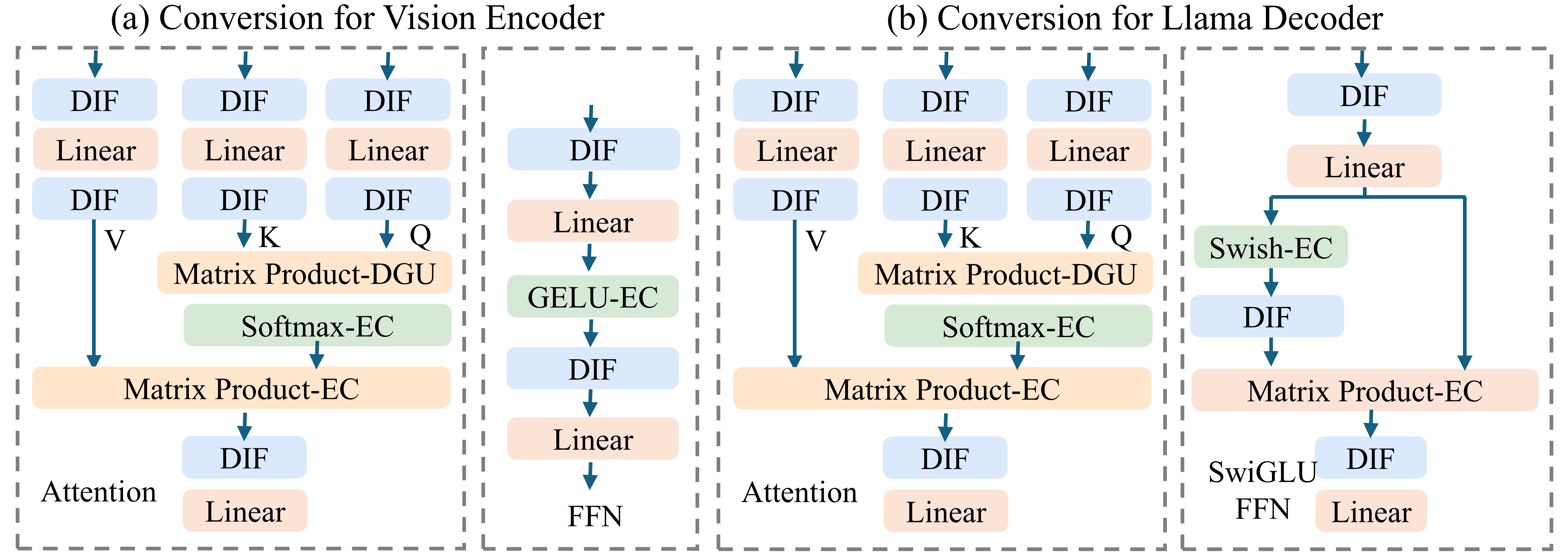}
    \caption{
    Component-wise differential conversion of the vision encoder and LLaMA decoder.
    (a) Vision encoder with DIF-based conversion of attention and feed-forward operations.
    (b) LLaMA decoder with RoPE-based attention and SwiGLU.
    Affine, nonlinear, and multiplicative operators are handled by their corresponding differential formulations.
    }
    \label{fig:placeholder}
    \vspace{-5pt}
\end{figure}

\subsection{Vision Encoder}
\label{app:vision_encoder}
\textbf{Execution Flow:} The vision encoder maps historical and current RGB observations into visual tokens for subsequent multimodal reasoning. We adopt the SigLIP backbone, which first converts each image into patch embeddings and augments them
with positional information. The resulting tokens are then processed by a stack of Transformer blocks, each consisting of normalization, multi-head self-attention (MHSA), feed-forward networks (FFNs), and residual connections. The final visual representations are subsequently fed into the multimodal projector. Formally, the computation of the $l$-th Transformer block is written as follows:
\begin{align}
\mathbf U^{l}
&=
\mathbf X^{l-1}
+
\operatorname{MHSA}
\!\left(
\operatorname{LN}(\mathbf X^{l-1})
\right),
 & 
\mathbf X^{l-1},\mathbf U^{l}
\in\mathbb{R}^{N_v\times d_v},
\label{eq:app_vision_attn}
\\
\mathbf X^{l}
&=
\mathbf U^{l}
+
\operatorname{FFN}
\!\left(
\operatorname{LN}(\mathbf U^{l})
\right),
& 
\mathbf U^{l},\mathbf X^{l}
\in\mathbb{R}^{N_v\times d_v},
\label{eq:app_vision_ffn}
\end{align}
where $\mathbf X^{l-1}$ and $\mathbf X^{l}$ denote the input and output visual tokens of the $l$-th block. We preserve this computation graph and convert its major operators individually as described below.

\textbf{Patch Projection:} The patch embedding maps each image patch into the visual hidden space. Since this operation is affine, its temporal increment is computed as follows:
\begin{equation}
\Delta\mathbf X^{0}[\tau]
=
\Delta\mathbf I[\tau]\widetilde{\mathbf W}_{\mathrm{patch}},
\label{eq:app_patch_proj}
\end{equation}
where $\widetilde{\mathbf W}_{\mathrm{patch}}$ denotes the reparameterized projection weight. Positional embeddings are fixed
within the internal SNN horizon and are incorporated into the initial visual token state.

\textbf{Layer Normalization:} LayerNorm is applied to the reconstructed activation state. Given the cached state $\mathbf X[\tau-1]$ and the incoming increment $\Delta\mathbf X[\tau]$, its differential response can be defined as follows:
\begin{equation}
\Delta\mathbf N[\tau]
=
\operatorname{LN}
\!\left(
\mathbf X[\tau-1]+\Delta\mathbf X[\tau]
\right)
-
\operatorname{LN}
\!\left(
\mathbf X[\tau-1]
\right).
\label{eq:app_vision_ln}
\end{equation}
The resulting temporal increment is then propagated to the subsequent operator.

\textbf{Self-Attention:} 
We preserve the original self-attention operators and apply DIF to the intermediate differential representations. Let \(\widetilde{\mathbf W}_X\) denote the basis-compensated projection weight for \(X\in\{Q,K,V\}\). The normalized increment is first encoded by DIF and projected into the query, key, and value branches, whose output increments are subsequently re-encoded before the attention interaction. These process are defined as follows:
\begin{equation}
\mathbf S_N[\tau]
=
\operatorname{DIF}_{\psi_N}\!\left(\Delta\mathbf N[\tau]\right),
\qquad
\mathbf S_X[\tau]
=
\operatorname{DIF}_{\psi_X}\!\left(
\mathbf S_N[\tau]\widetilde{\mathbf W}_X
\right),
\qquad
\widetilde{\mathbf W}_X
=
\mathbf Q_N^{-1}\mathbf W_X,\;
\label{eq:app_vision_qkv}
\end{equation}
Here, \(\widetilde{\mathbf W}_X\) denotes the reparameterized projection weight that compensates for the basis transformation introduced by the preceding DIF neuron. The DIF-encoded query and key increments are then propagated through the scaled dot-product attention. Following the bilinear differential rule, the attention-score increment is follows:
\begin{equation}
\Delta\mathbf A[\tau]
=
\frac{1}{\sqrt{d_h}}
\left(
\mathbf S_Q[\tau]\mathbf K^{\top}[\tau-1]
+
\mathbf Q[\tau-1]\mathbf S_K^{\top}[\tau]
+
\mathbf S_Q[\tau]\mathbf S_K^{\top}[\tau]
\right),
\label{eq:app_vision_qk}
\end{equation}
where \(d_h\) denotes the dimension of each attention head, and \(\mathbf Q[\tau-1]\) and \(\mathbf K[\tau-1]\) denote the accumulated query and key states from previous timesteps. The resulting attention-score increment is subsequently propagated through the original Softmax operator, whose temporal response is computed from consecutive score states. It can be defined as follows:
\begin{equation}
\Delta\mathbf P[\tau]
=
\operatorname{Softmax}
\!\left(
\mathbf A[\tau-1]+\Delta\mathbf A[\tau]
\right)
-
\operatorname{Softmax}
\!\left(
\mathbf A[\tau-1]
\right),
\quad \!\!
\mathbf S_P[\tau]
=
\operatorname{DIF}_{\psi_P}
\!\left(
\Delta\mathbf P[\tau]
\right).
\label{eq:app_vision_softmax}
\end{equation}
The DIF-encoded attention probability increment is then propagated through the value interaction. Following the same bilinear differential rule, the resulting attention-output increment is defined as:
\begin{equation}
\Delta\mathbf O[\tau]
=
\mathbf S_P[\tau]\mathbf V[\tau-1]
+
\mathbf P[\tau-1]\mathbf S_V[\tau]
+
\mathbf S_P[\tau]\mathbf S_V[\tau],
\label{eq:app_vision_pv}
\end{equation}
where \(\mathbf P[\tau-1]\) and \(\mathbf V[\tau-1]\) denote the accumulated attention probability and value states, respectively.
The resulting increment is encoded by DIF before the output projection. Using the same basis-compensated affine mapping, these process is are defined as follows:
\begin{equation} 
\mathbf S_O[\tau]
=
\operatorname{DIF}_{\psi_O}
\!\left(
\Delta\mathbf O[\tau]
\right),
\qquad
\Delta\mathbf H_{\mathrm{attn}}[\tau]
=
\mathbf S_O[\tau]\widetilde{\mathbf W}_O,
\qquad
\widetilde{\mathbf W}_O
=
\mathbf Q_O^{-1}\mathbf W_O .
\label{eq:app_vision_output}
\end{equation}
The resulting attention increment is propagated through the original residual connection by direct temporal accumulation. We next apply the differential conversion principle to the FFN.

\textbf{Feed-Forward Network:} The FFN preserves its original affine--nonlinear structure. For a pointwise activation $\phi(\cdot)$, its temporal increment is computed from consecutive input states:
\begin{equation}
\Delta\mathbf Y_{\phi}[\tau]
=
\phi\!\left(
\mathbf X[\tau-1]+\Delta\mathbf X[\tau]
\right)
-
\phi\!\left(
\mathbf X[\tau-1]
\right).
\label{eq:app_vision_act}
\end{equation}
The surrounding affine projections use basis-compensated weights, while each nonlinear increment is encoded by DIF before entering the subsequent projection. The resulting FFN increment is accumulated through the residual connection to update the visual token state.

\subsection{Multimodal Projector}
\label{app:projector}
\textbf{Projection:} The multimodal projector maps the visual representations into the hidden space of the language decoder. We retain the original MLP structure and apply the same differential conversion to its affine and nonlinear
operators. For a two-layer projector, the computation can be expressed as:
\begin{equation}
\mathbf S_V[\tau]
=
\operatorname{DIF}_{\psi_V}
\!\left(
\Delta\mathbf V[\tau]
\right),
\qquad
\mathbf S_H[\tau]
=
\operatorname{DIF}_{\psi_H}
\!\left(
\mathbf S_V[\tau]\widetilde{\mathbf W}_1
\right),
\qquad
\Delta\mathbf Z[\tau]
=
\mathbf S_G[\tau]\widetilde{\mathbf W}_2,
\label{eq:app_projector}
\end{equation}
$\mathbf S_G[\tau]$ denotes the DIF-encoded differential response of the intermediate activation. Its nonlinear response follows the pointwise differential formulation in Eq.~\ref{eq:app_vision_act}. Thus, the original projector architecture is preserved while its visual representations are propagated in differential form.

\textbf{Token Assembly:}
The projected historical and current visual tokens are concatenated with the language embedding to construct the multimodal input sequence. Since the instruction remains static within the SNN inference horizon, its embedding is included in the initial state and contributes no temporal increment. The multimodal increment is therefore formulated as:
\begin{equation}
\Delta\mathbf X_0[\tau]
=
\left[
\Delta\mathbf Z^{h}[\tau];
\Delta\mathbf Z^{c}[\tau];
\mathbf 0
\right],
\label{eq:app_token_assembly}
\end{equation}
$\Delta\mathbf Z^{h}[\tau]$ and $\Delta\mathbf Z^{c}[\tau]$ denote the projected increments of historical and current visual tokens. Since concatenation only reorganizes token ordering, it introduces no additional differential operation.

\subsection{LLaMA Decoder}
\label{app:llama_decoder}
\textbf{Execution Flow:} The LLaMA decoder takes the assembled visual and language tokens as input and performs multimodal reasoning to generate the high-level navigation response. The multimodal sequence is processed by a stack of Transformer decoder blocks, each following a pre-normalization architecture with RMSNorm, RoPE-based self-attention, SwiGLU, and residual connections.
Formally, the $l$-th decoder block is written as follows:
\begin{align}
\mathbf U^{l}
&=
\mathbf H^{l-1}
+
\operatorname{Attn}_{\mathrm{RoPE}}
\!\left(
\operatorname{RMSNorm}(\mathbf H^{l-1})
\right),
&
\mathbf H^{l-1},\mathbf U^{l}
&\in\mathbb{R}^{N_m\times d_m},
\label{eq:app_llama_attn}
\\
\mathbf H^{l}
&=
\mathbf U^{l}
+
\operatorname{SwiGLU}
\!\left(
\operatorname{RMSNorm}(\mathbf U^{l})
\right),
&
\mathbf U^{l},\mathbf H^{l}
&\in\mathbb{R}^{N_m\times d_m}.
\label{eq:app_llama_ffn}
\end{align}
Here, $\mathbf H^{l-1}$ and $\mathbf H^{l}$ denote the input and output
multimodal representations of the $l$-th decoder block, respectively,
while $N_m$ and $d_m$ denote the multimodal sequence length and hidden
dimension. The attention module applies rotary positional embeddings to
the query and key representations before scaled dot-product attention,
whereas SwiGLU combines the SiLU-activated gate branch with the up-projection
branch through element-wise multiplication. We retain this computation
graph and apply operation-aware differential conversion to its constituent
operators below.

\textbf{RoPE-Based Self-Attention:}
For the query, key, and value projections, the normalized increment is encoded by DIF and then propagated through the corresponding basis-compensated affine mappings:
\begin{equation}
\mathbf S_N[\tau]
=
\operatorname{DIF}_{\psi_N}
\!\left(
\Delta\mathbf N[\tau]
\right),
\qquad
\mathbf S_X[\tau]
=
\operatorname{DIF}_{\psi_X}
\!\left(
\mathbf S_N[\tau]\widetilde{\mathbf W}_X
\right),
\quad
X\in\{Q,K,V\}.
\label{eq:app_llama_qkv}
\end{equation}
RoPE applies a position-dependent linear rotation to the query and key representations. Since token positions remain fixed over the internal SNN horizon, the corresponding rotations are constant across timesteps and can be applied directly to the differential representations. Thus, RoPE requires no additional differential operator and is omitted from the subsequent notation for clarity. The subsequent query--key interaction, causal Softmax, attention--value interaction, and output projection follow the bilinear, nonlinear, and basis-compensated differential formulations.
Since the causal mask remains fixed over the internal SNN horizon, it is directly retained in the Softmax computation.

\textbf{SwiGLU Feed-Forward Network:} SwiGLU contains parallel gate and up-projection branches followed by an element-wise interaction. Their projection increments are obtained using the basis-compensated affine mappings. For the gate branch, the SiLU
response is follows:
\begin{equation}
\Delta\mathbf G[\tau]
=
\operatorname{SiLU}
\!\left(
\mathbf Z_g[\tau-1]+\Delta\mathbf Z_g[\tau]
\right)
-
\operatorname{SiLU}
\!\left(
\mathbf Z_g[\tau-1]
\right),
\qquad
\mathbf S_G[\tau]
=
\operatorname{DIF}_{\psi_G}
\!\left(
\Delta\mathbf G[\tau]
\right),
\end{equation}
while the up-projection increment is similarly encoded as $\mathbf S_U[\tau]$. The two encoded branches are then combined through
the multiplicative differential rule:
\begin{equation}
\Delta\mathbf M[\tau]
=
\mathbf S_G[\tau]\odot\mathbf U[\tau-1]
+
\mathbf G[\tau-1]\odot\mathbf S_U[\tau]
+
\mathbf S_G[\tau]\odot\mathbf S_U[\tau].
\label{eq:app_llama_swiglu}
\end{equation}
The resulting increment is encoded by DIF and propagated through the basis-compensated down projection before residual accumulation.
\subsection{Low-Level Actor}
\label{app:low_level_actor}
\textbf{Execution Flow:} The low-level actor maps the current locomotion state and high-level navigation command to the desired joint positions. We retain the original MLP architecture, consisting of three hidden layers with ELU activations followed by a linear output layer. Let $\mathbf h^{0}$ denote the actor input. The $l$-th hidden layer and the final joint prediction are written
as follows:
\begin{equation}
\mathbf h^{l}
=
\operatorname{ELU}
\!\left(
\mathbf h^{l-1}\mathbf W_l+\mathbf b_l
\right),
\quad l=1,\ldots,L-1,
\qquad
\mathbf q^{d}
=
\mathbf h^{L-1}\mathbf W_L+\mathbf b_L .
\label{eq:app_actor_flow}
\end{equation}
The original actor computation is preserved, while its affine and nonlinear operators are converted using the differential formulations described below.

\textbf{Differential Actor Conversion:} The actor consists of affine layers and pointwise ELU activations, which are converted using the single-input differential formulation. For the \(l\)-th hidden layer, the affine increment is propagated through the basis-compensated weight and then evaluated by the differential ELU response. These process is defined as follows:
\begin{equation}
\Delta\mathbf Z^{l}[\tau]
=
\mathbf S^{l-1}[\tau]\widetilde{\mathbf W}_{l},
\quad
\Delta\mathbf H^{l}[\tau]
=
\operatorname{ELU}
\!\left(
\mathbf Z^{l}[\tau-1]+\Delta\mathbf Z^{l}[\tau]
\right)
-
\operatorname{ELU}
\!\left(
\mathbf Z^{l}[\tau-1]
\right).
\label{eq:app_actor_hidden}
\end{equation}
Thus, the low-level action is progressively refined as differential responses are accumulated over internal timesteps, while the original actor and downstream joint-control interfaces remain unchanged.

\newpage
\section{Visualization of Navigation Trajectories}
\label{app:qualitative}
We provide additional qualitative examples to evaluate SpikingVLA across diverse navigation scenarios. As shown in Fig.~\ref{fig:app_navigation_examples}, SpikingVLA generates coherent action sequences across different environments, instructions, and trajectories. These results complement the quantitative evaluation by demonstrating consistent navigation behavior on representative unseen episodes.

\begin{figure}[!htbp]
    \centering
    \includegraphics[width=0.98\linewidth]{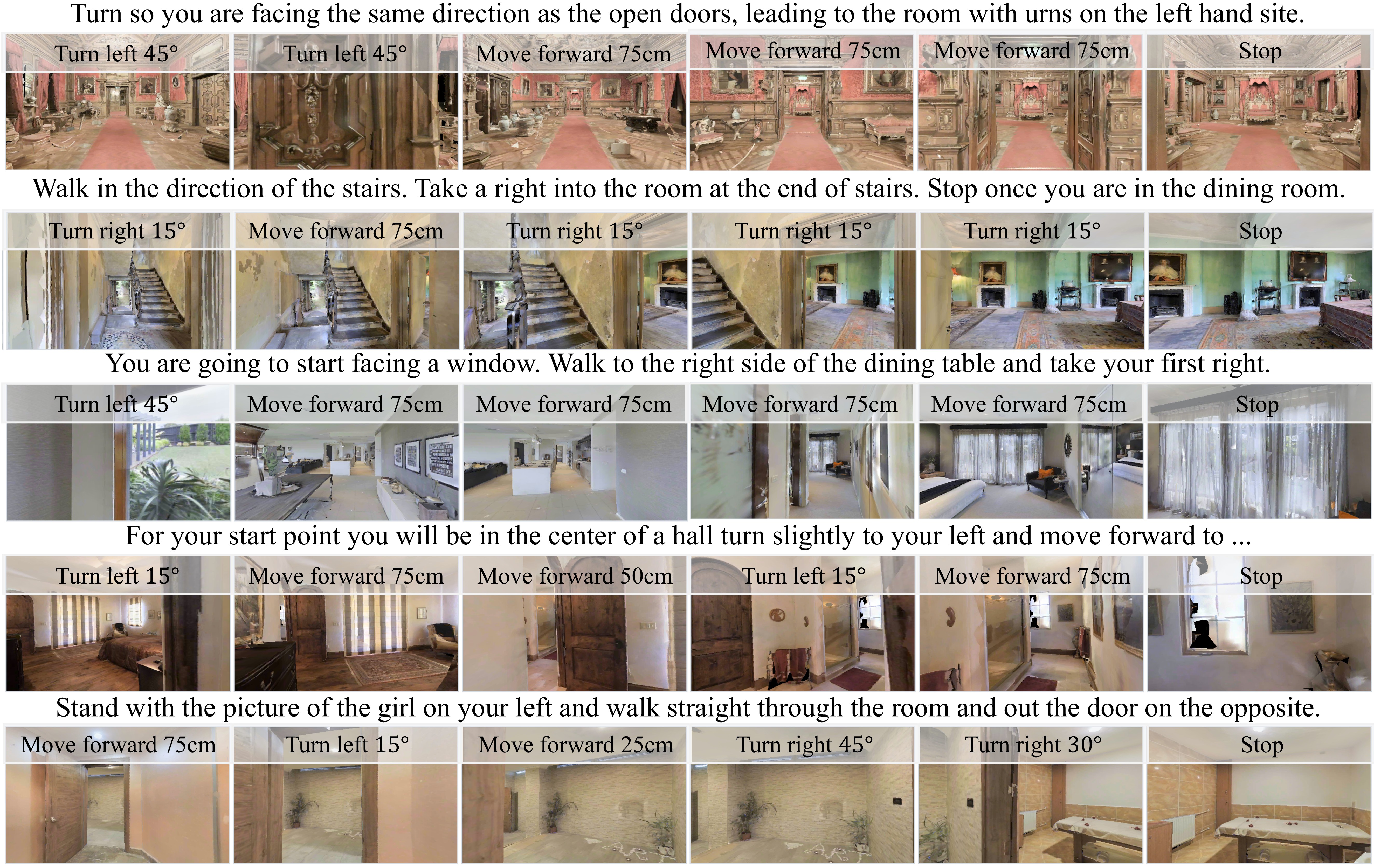}
    \caption{Representative navigation rollouts of SpikingVLA.}
    \label{fig:app_navigation_examples}
\end{figure}
We further examine the temporal evolution of action predictions over internal SNN timesteps. As shown in Fig.~\ref{fig:app_action_refinement}, predictions progressively approach the corresponding ANN decisions as differential responses accumulate. Different inputs exhibit distinct convergence rates, indicating input-dependent temporal requirements for action generation.

\begin{figure}[!htbp]
    \centering
    \includegraphics[width=0.98\linewidth]{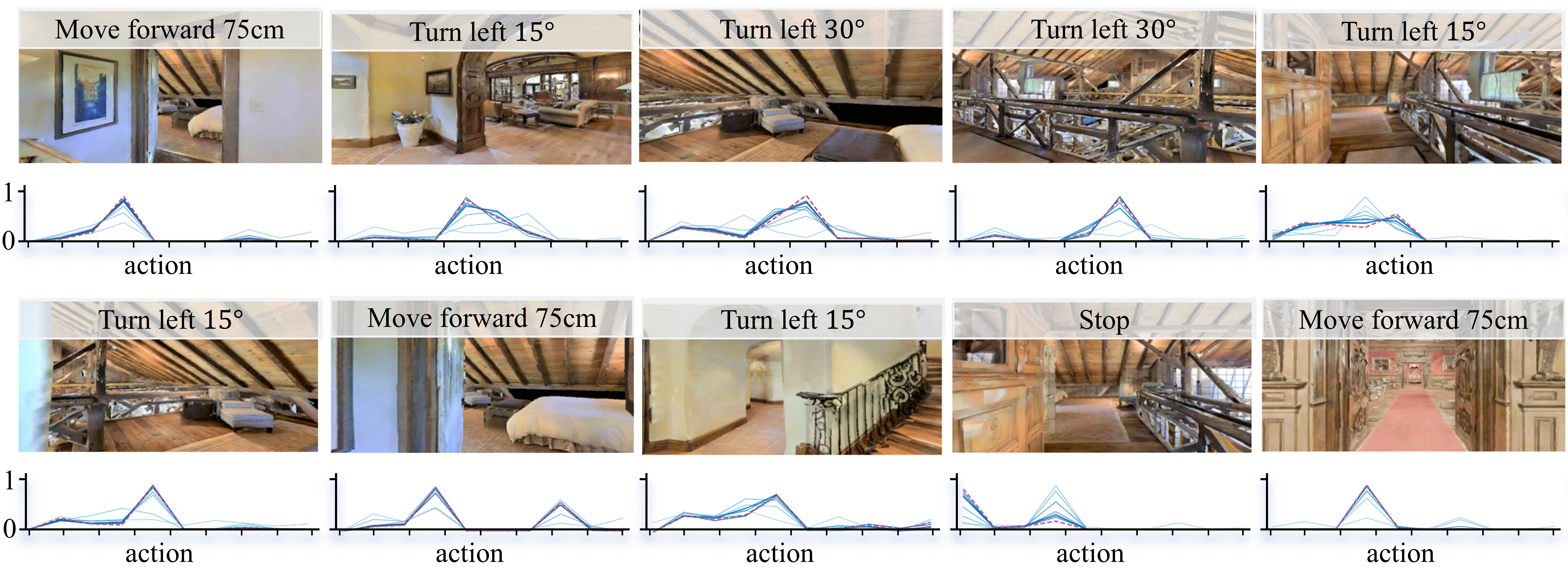}
    \caption{Temporal convergence of action predictions toward the corresponding ANN outputs across internal timesteps, with different inputs exhibiting distinct convergence rates.
    }
    \label{fig:app_action_refinement}
\end{figure}

To quantify this temporal behavior, we record the first timestep at which
each sample satisfies the stopping criterion. With a maximum horizon of
$T=8$, SpikingVLA requires only $4.1$ timesteps on average, corresponding
to a $48.8\%$ reduction in the effective temporal budget. The variation
in stopping times further indicates that different inputs require
different amounts of temporal evidence, allowing unnecessary updates to
be avoided once the action prediction has stabilized.

\end{document}

%% file: math_commands.tex
\usepackage{amsmath,amsfonts,bm}

\def\eqref#1{equation~\ref{#1}}

\def\1{\bm{1}}

\DeclareMathAlphabet{\mathsfit}{\encodingdefault}{\sfdefault}{m}{sl}
\SetMathAlphabet{\mathsfit}{bold}{\encodingdefault}{\sfdefault}{bx}{n}

